\documentclass{scrartcl}[11pt]
\usepackage[utf8]{inputenc} 
\usepackage[T1]{fontenc}    
\usepackage{amsfonts}       
\usepackage{amsmath}
\usepackage{amssymb}
\usepackage{mathtools}
\usepackage{amsthm}
\usepackage{bm}
\usepackage[usenames]{xcolor}
\usepackage{soul}           
\usepackage{url}
\usepackage{booktabs}       
\usepackage{nicefrac}       
\usepackage{enumerate}
\usepackage{graphicx}
\DeclareGraphicsExtensions{.pdf,.png,.jpg}
\graphicspath{{figures/}}
\usepackage{natbib}
\usepackage{comment}
\usepackage{hyperref}
\usepackage{centernot}     

\usepackage{fullpage}
\usepackage[titletoc,title]{appendix}
\usepackage{microtype}      
\newcommand{\R}{\mathbb{R}}

\newcommand{\BigO}[1]{\ensuremath{\mathcal{O}\left(#1\right)}}                  
\newcommand{\BigOm}[1]{\ensuremath{\Omega\left(#1\right)}}                      

\newcommand{\vect}[1]{\ensuremath{\mathbf{#1}}}                                 
\newcommand{\vectsym}[1]{\ensuremath{\boldsymbol{#1}}}                          
\newcommand{\mat}[1]{\ensuremath{\mathbf{\MakeUppercase{#1}}}}                  
\newcommand{\KL}[2]{\ensuremath{\mathbb{KL} \left(#1 \middle\Vert #2 \right)}}   
\newcommand{\Exp}[2]{\ensuremath{\mathbb{E}_{#1}\left[#2\right]}}                
\newcommand{\Tr}[1]{\ensuremath{\mathrm{Tr}\left(#1\right)}}                     
\newcommand{\Ind}[1]{\ensuremath{\mathbf{1}\left[#1\right]}}                     
\newcommand{\Norm}[1]{\ensuremath{\lVert #1 \rVert}}                  
\newcommand{\NormII}[1]{\ensuremath{\lVert #1 \rVert}_2}              
\newcommand{\aNormII}[1]{\ensuremath{\left\lVert #1 \right\rVert}_2}              
\newcommand{\NormInfty}[1]{\ensuremath{\lVert #1 \rVert_{\infty}}}    
\newcommand{\parg}{\makebox[1ex]{$\mathbf{\cdot}$}}                                

\newcommand{\matrx}[1]{\begin{bmatrix}#1\end{bmatrix}}                           

\newcommand{\InNorm}[1]{{\left\vert\kern-0.2ex\left\vert\kern-0.2ex\left\vert #1 
    \right\vert\kern-0.2ex\right\vert\kern-0.2ex\right\vert}}                    

\newcommand{\InNormII}[1]{{\left\vert\kern-0.2ex\left\vert\kern-0.2ex\left\vert #1 
    \right\vert\kern-0.2ex\right\vert\kern-0.2ex\right\vert}_2}                    

\newcommand{\InNormInfty}[1]{{\left\vert\kern-0.2ex\left\vert\kern-0.2ex\left\vert #1 
    \right\vert\kern-0.2ex\right\vert\kern-0.2ex\right\vert}_{\infty}}           

\newcommand{\Abs}[1]{\ensuremath{\lvert #1  \rvert}}                             
\newcommand{\Prob}[1]{\ensuremath{\mathrm{Pr}\left\{ #1 \right\}}}               
\newcommand{\iid}{i.i.d~}                                                        
\newcommand{\Grad}{\nabla}                                                       
\DeclarePairedDelimiterX{\Inner}[2]{\langle}{\rangle}{#1, #2}                    

\newcommand{\MI}{\mathnormal{I}}                                                     

\newcommand{\Land}{\wedge}                                                       
\newcommand{\defeq}{\overset{\mathrm{def}}{=}}                                                      

\DeclareMathOperator*{\union}{\cup}

\DeclareMathOperator*{\argmin}{argmin}
\DeclareMathOperator*{\argmax}{argmax}

\newtheorem{lemma}{Lemma}
\newtheorem{theorem}{Theorem}
\newtheorem{remark}{Remark}

\theoremstyle{definition}

\newcommand{\pdv}[2]{\frac{\partial \mkern2mu #1}{\partial \mkern1mu #2}}     
\newcommand{\ppdv}[2]{\frac{\partial^2 \mkern2mu #1}{\partial \mkern1mu {#2}^2}}     
\newcommand{\ipdv}[2]{\nicefrac{\partial \mkern2mu #1}{\partial \mkern1mu #2}}     

\newcommand{\mA}{\mat{A}}
\newcommand{\mB}{\mat{B}}

\newcommand{\mH}{\mat{H}}

\newcommand{\mR}{\mat{R}}

\newcommand{\mX}{\mat{X}}
\newcommand{\mY}{\mat{Y}}

\newcommand{\va}{\vect{a}}
\newcommand{\vb}{\vect{b}}
\newcommand{\vc}{\vect{c}}

\newcommand{\vg}{\vect{g}}

\newcommand{\vx}{\vect{x}}
\newcommand{\vy}{\vect{y}}

\newcommand{\Set}[1]{\{#1\}}
\newcommand{\aSet}[1]{\left\{#1\right\}}
\newcommand{\Pf}{\mathcal{P}}               
\newcommand{\Gm}{\mathcal{G}}  
\newcommand{\A}{\mathcal{A}}   
\newcommand{\Nb}{\mathcal{N}}   
\newcommand{\mi}{-i}            
\newcommand{\U}{\mathcal{U}}    
\newcommand{\NE}{\mathcal{NE}}  
\newcommand{\eNE}{\varepsilon\text{-}\mathcal{NE}}  
\newcommand{\Data}{\mathcal{D}} 
\newcommand{\Gmh}{\widehat{\Gm}}  
\newcommand{\Uh}{\widehat{\U}}     
\newcommand{\Gh}{\widehat{G}}     
\newcommand{\vtheta}{\vectsym{\theta}}             
\newcommand{\vbtheta}{\vectsym{\bar{\theta}}}      
\newcommand{\vhtheta}{\vectsym{\widehat{\theta}}}  

\newcommand{\vf}{\vect{f}}
\newcommand{\loss}{\ell}
\newcommand{\eigmax}{\lambda_{\mathrm{max}}}
\newcommand{\eigmin}{\lambda_{\mathrm{min}}}
\newcommand{\Ft}{\widetilde{F}}  				
\newcommand{\DelS}{\vectsym{\Delta}_{S}}                   
\newcommand{\DelhS}{\widehat{\vectsym{\Delta}}_{S}}        
\newcommand{\vDel}{\vectsym{\Delta}}                       
\newcommand{\cmin}{C_{\mathrm{min}}}                       
\newcommand{\Sb}{\bar{S}}                                  
\newcommand{\Sbc}{\bar{S^c}}                               
\newcommand{\DelSb}{\vDel_{\Sb}}                           
\newcommand{\DelSbc}{\vDel_{\Sbc}}                         
\newcommand{\gNormIII}[1]{\Norm{#1}_{1,2}}                 
\newcommand{\gNormInII}[1]{\Norm{#1}_{\infty,2}}           
\newcommand{\uh}{\widehat{u}}							  
\newcommand{\mcp}[1]{\mathclap{#1}}                        
\newcommand{\sGm}{\mathfrak{G}}                            
\newcommand{\stGm}{\widetilde{\mathfrak{G}}}               
\newcommand{\cI}{\mathcal{I}}                              
\newcommand{\perr}{p_{\mathrm{err}}}                       
\newcommand{\Maj}{\mathfrak{maj}}                          

\definecolor{Cons}{HTML}{FF5C62} 
\definecolor{Lib}{HTML}{74B4FF} 
\definecolor{Neu}{HTML}{86FF8A} 
\newcommand{\Cons}[1]{\textbf{\colorbox{Cons}{{\color{Cons}|}#1{\color{Cons}|}}}}
\newcommand{\Lib}[1]{\textbf{\colorbox{Lib}{{\color{Lib}|}#1{\color{Lib}|}}}}
\newcommand{\Neu}[1]{\textbf{\colorbox{Neu}{{\color{Neu}|}#1{\color{Neu}|}}}}

\begin{document}

\title{Learning Sparse Polymatrix Games in Polynomial Time and Sample Complexity}

\author{Asish Ghoshal and Jean Honorio\\
Department of Computer Science\\
Purdue University\\
West Lafayette, IN - 47906\\
\{aghoshal, jhonorio\}@purdue.edu}

\date{}

\maketitle

\begin{abstract}
We consider the problem of learning sparse polymatrix games from observations of strategic interactions.
We show that a polynomial time method based on $\ell_{1,2}$-group regularized logistic regression
recovers a game, whose Nash equilibria are the $\epsilon$-Nash equilibria of the game from which the data was generated (true game),
in $\BigO{m^4 d^4 \log (pd)}$ samples of strategy profiles ---
where $m$ is the maximum number of pure strategies of a player, $p$ is the number of players,
and $d$ is the maximum degree of the game graph. Under slightly more stringent separability conditions on the payoff
matrices of the true game, we show that our method learns a game with the exact same Nash equilibria as the true game.
We also show that $\BigOm{d \log (pm)}$ samples are necessary for any method 
to consistently recover a game, with the same Nash-equilibria as the true game, from observations of strategic interactions.
We verify our theoretical results through simulation experiments.
\end{abstract}
\section{Introduction and Related Work}
\paragraph{Motivation.}
Many complex real-world data can be thought of as resulting from the 
behavior of a large number of self-interested trying to myopically or locally maximize some utility.
Over the past several decades, non-cooperative game theory has emerged
as a powerful mathematical framework for reasoning about
such strategic interactions between self-interested agents. Traditionally,
research in game theory has focused on computing the
\emph{Nash equilibria} (NE) (c.f. \cite{blum2006continuation} and \cite{jiang2011polynomial}) ---
 which characterizes the stable outcome of the overall behavior of self-interested agents ---
\emph{correlated equilibria} (c.f. \citep{kakade2003correlated}), and other solution concepts given 
a description of the game. Computing the \emph{price of anarchy} (PoA) for graphical games,   
which in a sense quantifies the \emph{inefficiency of equilibria}, 
is also of tremendous interest (c.f. \citep{ben2011local}).
The aforementioned problems of computing the NE, correlated equilibria and PoA
can be thought of as \emph{inference problems} in graphical games, and require a description of
the game, i.e., the payoffs of the players.
In many real-world settings, however, only the behavior of the agents are observed,
in which case inferring the latent payoffs of the players from observations of behavioral data becomes imperative. 
\emph{This problem of learning a game from observations of behavioral data, i.e.,
recovering the structure and parameters of the player payoffs
such that the Nash equilibria of the game, in some sense, 
approximates the Nash equilibria of the true game, is the primary focus of the paper.}

Recovering the underlying game from behavioral data is an important tool in 
exploratory research in political science and behavioral economics,
and recent times have seen a surge of interest in such problems (c.f. \citep{irfan14, honorio15, ghoshal2016, garg_learning_2016, ghoshal2017learning}). For instance, in political science, \cite{irfan14} identified the \emph{most influential} senators in the
U.S congress --- a small coalition of senators whose collective behavior
forced every other senator to a unique choice of action ---
by learning a \emph{linear influence game} from congressional voting records.
\cite{garg_learning_2016} showed that a \emph{tree-structured
polymatrix game} \footnote{\cite{garg_learning_2016} call their game 
a \emph{potential game} even though the formulation of their game is similar to ours.} learned from U.S. Supreme Court data was able to recover
the known ideologies  of the justices. However, many open problems remain
in this area of active research. One such problem is whether there exists
efficient (polynomial time) methods for learning polymatrix games \citep{Janovskaja68} from noisy observations
of strategic interactions. This is the focus of the current paper.

\paragraph{Related Work.}
Various methods have been proposed for learning games from data. \cite{honorio15}
proposed a maximum-likelihood approach to learn ``linear influence games'' ---
a class of parametric graphical games with linear payoffs. However, in addition
to being exponential time, the maximum-likelihood approach of \cite{honorio15}
also assumed a specific observation model for the strategy profiles. \cite{ghoshal2016} proposed
a polynomial time algorithm, based on $\ell_1$-regularized logistic regression,
for learning linear influence games. They again assumed the specific observation model
proposed by \cite{honorio15} in which the strategy profiles (or joint actions) were
drawn from a mixture of uniform distributions: one over the pure-strategy Nash equilibria (PSNE) set,
and the other over the complement of the PSNE set. \cite{ghoshal2017learning} obtained
necessary and sufficient conditions for learning linear influence games under arbitrary observation model.
Finally, \cite{garg_learning_2016} use a discriminative, max-margin based approach,
to learn tree structured polymatrix games. However, their method is exponential time and they show
that learning polymatrix games is NP-hard under this max-margin setting,
even when the class of graphs is restricted to trees. Furthermore, all the aforementioned works,
with the exception of \cite{garg_learning_2016}, consider binary strategies only. 
In this paper, we propose a polynomial time algorithm for learning polymatrix games,
which are non-parametric graphical games where the pairwise payoffs between players are characterized by matrices
(or pairwise potential functions). In this setting, each player has a finite number of pure-strategies.

\paragraph{Our Contributions.} We propose an $\ell_{1,2}$ group-regularized logistic regression
method to learn polymatrix games, which has been considered by \cite{garg_learning_2016}
and is a generalization of linear influence games considered by \cite{ghoshal2017learning}.
We make no assumptions on the latent payoff functions and show that our polynomial time
algorithm recovers an $\varepsilon$-Nash equilibrium of the true game \footnote{By the phrase 
``recovering the Nash equilibria'' we mean that we learn a game with the same Nash
equilibria as the true game. We use this phrase elsewhere in the paper for brevity.}, 
with high probability,
if the number of samples is $\BigO{m^4 d^4 \log (p d)}$, where $p$ is the number of players,
$d$ is the maximum degree of the game graph and $m$ is the maximum number of pure-strategies of a player.
Under slightly more stringent separability conditions on the payoff functions of the underlying game, we show
that our method recovers the Nash equilibria set exactly. We further generalize the observation
model from \cite{ghoshal2017learning} in the sense that we allow strategy profiles in the non-Nash equilibria set
to have zero measure. This should be compared with the results of \cite{garg_learning_2016} who
show that learning tree-structured polymatrix games is NP-hard under a max-margin setting. We also obtain necessary
conditions on learning polymatrix games and show that $\BigOm{d \log (pm)}$ samples are required
by any method for recovering the PSNE set of a polymatrix game from observations of strategy profiles.

Finally, we conclude this section by referring the reader to the work of \cite{jalali_learning_2011}
who analyze $\ell_{1,2}$-regularized logistic regression for learning undirected graphical models. However,
our setting differs from that of learning discrete graphical models in many ways. First, unlike discrete
graphical models, where the underlying distribution over the variables is described by a potential function
that factorizes over the cliques of the graph, we make no assumptions whatsoever on the generative distribution
of data. Further, we are interested in recovering the PSNE set of a game, since the graph structure in generally
unidentifiable from observational data, whereas \cite{jalali_learning_2011} obtain guarantees on the graph structure
of the discrete graphical model. As a result, our theoretical analysis and
proofs differ significantly from those of \cite{jalali_learning_2011}.
\section{Notation and Problem Formulation}
In this section, we introduce our notation and formally define the problem
of learning polymatrix games from behavioral data.

\paragraph{Polymatrix games.} A $p$-player \emph{polymatrix game} is a graphical game where the set of nodes
of the graph denote players and the edges correspond to two-player games. 
We will denote the graph by $G = ([p], E)$, where $[p] \defeq \Set{1, \ldots, p}$ is 
the vertex set and $E \subseteq [p] \times [p]$ is set of directed edges. 
An edge $(i,j) \in E$ denotes the directed edge $i \leftarrow j$.
Each player $i$ has a set of pure-strategies or actions $\A_i$, 
and the set of pure-strategy profiles or joint actions of all the $p$ players
is denoted by $\A = \times_{i \in [p]} \A_i$. We will denote $\A_{\mi} \defeq \times_{j \in \mi} \A_j$.
With each edge $(i,j) \in E$ is associated 
a payoff matrix $u^{i,j}: \A_i \times \A_j \rightarrow \R$, such that $u^{i,j}(x_i, x_j)$
gives the finite payoff of the $i$-th player (with respect to the $j$-th player),
when player $i$ plays $x_i \in \A_i$ and player
$j$ plays $x_j \in \A_j$. We assume that $(i,j) \in E$, 
if and only if $u^{i,j}(\parg, \parg) \neq 0$.
Given a strategy profile $\vx \in \A$,
the total payoff, or simply the payoff, of the $i$-th player 
is given by the following potential function:
\begin{align}
 u^i(x_i, \vx_{\mi}; G) = u^{i,i}(x_i) + \sum_{j \in \Nb_i} u^{i,j}(x_i, x_j), \label{eq:payoff}
\end{align}
where $\Nb_i(G) \defeq \Set{j \in [p] | (i,j) \in E}$ is the set of neighbors of $i$ in the graph $G$,
and $u^{i,i}: \A_i \rightarrow \R$ gives the (finite) \emph{individual payoff} of $i$ for playing $x_i$. 
We will denote the number of neighbors of player $i$ by $d_i \defeq \Abs{\Nb_i(G)}$, and 
the maximum degree of the graph $G$ by $d = \max\Set{d_1, \ldots, d_p}$.
A polymatrix game $\Gm = (G, \U)$ is then completely defined by a graph $G = ([p], E)$ and a collection
of potential functions  
$\U(G) = \Set{u^{i}: \A_{\mi} \rightarrow \R}_{i\in[p]}$, 
where each of the payoff functions $u^i(\parg ; G)$
decomposes according to \eqref{eq:payoff}. Finally, we will also assume that the number of 
strategies of each player, $m_i \defeq \Abs{\A_i}$, is non-zero 
and $\BigO{1}$ with respect to $p$ and $d$, and that $m \defeq \max\{ m_i \}$.

\paragraph{Nash equilibria of polymatrix games.}
The pure-strategy Nash equilibria (PSNE) set for the game $\Gm = (G, \U)$ is given by the set of strategy profiles where
no player has any incentive to unilaterally deviate from its strategy given the strategy profiles of its neighbors,
and is defined as follows:
\begin{align}
\NE(\Gm) = \aSet{ \vx \in \A \mathrel{\Big|} x_i \in \argmax_{a \in \A_i} u^i(a, \vx_{\mi}) }.
\end{align}
The set of $\varepsilon$-Nash equilibria of the game $\Gm$ are those strategy profiles where each player
can gain at most $\varepsilon$ payoff by deviating from its strategy, and is defined as follows:
\begin{align}
\eNE(\Gm) &= \Bigl\{ \vx \in \A \mid 
	u^i(x_i, \vx_{\mi}) \geq u^i(a, \vx_{\mi}) - \varepsilon,  \forall a \in \A_i \text{ and } \forall i \in [p] \Bigr\}.
\end{align}

\paragraph{Observation model.}
Without getting caught up in the dynamics of gameplay --- something that is difficult to observe or
reason about in real-world scenarios --- we abstract the learning problem as follows. Assume that we are
given ``noisy'' observations of strategy profiles, or joint actions, $\Data = \Set{\vx^{(l)} \in \A}_{l \in [n]}$
 drawn from a game $\Gm = (G, \U)$. The noise process models our uncertainty over the 
individual actions of the players due to observation noise,
for instance, when we observe the actions through a noisy channel, 
or due to the unobserved dynamics of gameplay during which equilibrium is reached. 
By ``observations drawn from a game'' we simply mean that there exists a distribution $\Pf$, from
which the strategy profiles are drawn, satisfying the following condition:
\begin{gather*}
\forall \vx, \vx' \text{ such that } \vx \in \NE(\Gm) \text{ and } \vx' \in \A \setminus \NE(\Gm): 
 \Pf(\vx) > \Pf(\vx').
\end{gather*}
The above condition ensures that the signal level is more than the noise level. 
This should  be compared with the observation model of \cite{ghoshal2017learning}, who assume that
$\forall \vx' \in \A \setminus \NE(\Gm), \Pf(\vx') > 0$. Our observation model thus encompasses
specific observation models considered in prior literature \citep{honorio15, ghoshal2016}: the global and local noise model.
The global noise model is parameterized by a constant $q \in (\nicefrac{\NE(\Gm)}{\Abs{\A}}, 1)$ such that the probability
of observing a strategy profile $\vx \in \A$ is given by a mixture of two uniform distributions:
\begin{align}
\Pf_g(\vx; \Gm) = q \frac{\Ind{\vx \in \NE(\Gm)}}{\Abs{\NE(\Gm)}} +
	(1 - q) \frac{\Ind{\vx \notin \NE(\Gm)}}{\Abs{\A} - \Abs{\NE(\Gm)}}. \label{eq:global_noise}
\end{align}
In the local noise model, we observe strategy profiles $\vx$ from the PSNE set with each entry (strategy)
corrupted independently. Therefore, in the local noise model we have the following 
distribution over strategy profiles:
\begin{align}
\Pf_l(\vx; \Gm) = \frac{1}{\Abs{\NE(\Gm)}} \times 
	 \sum_{\vy \in \NE(\Gm)}  \prod_{i=1}^p (q_i)^{\Ind{x_i = y_i}} \left(\frac{1 - q_i}{m_i - 1}\right)^{\Ind{x_i \neq y_i}},
\label{eq:local_noise}
\end{align}
with $q_i > 0.5$ for all $i \in [p]$.

In essence, we assume that we observe multiple ``stable outcomes'' of the game, which may or may-not be in equilibria.
Treating the outcomes of the game as ``samples'' observed across multiple ``plays'' of the same game is a recurring
theme in the literature for learning games (c.f. \citep{honorio15}, \citep{ghoshal2016}, \citep{ghoshal2017learning},
\citep{garg_learning_2016}).

The learning problem then
corresponds to recovering a game $\Gmh = (\Gh, \Uh)$ from $\Data$ such that $\NE(\Gmh) = \NE(\Gm)$ with
high probability. \emph{Given that computing a single Nash equilibria is PPAD-complete \citep{daskalakis_complexity_2009},
any efficient learning algorithm must learn the game without explicitly computing or enumerating the Nash equilibria of the game.}
It has also been shown that even computing an $\varepsilon$-Nash equlibria is hard under the exponential time hypothesis for PPAD 
\citep{rubinstein_settling_2016}. We also emphasize that we do not observe any information about the latent player payoffs, 
and neither do we impose any restrictions on the payoffs for obtaining our $\varepsilon$-Nash equilibria guarantees.
Also, note that in our definition of the learning problem, 
we do not impose any restriction on the ``closeness'' of the recovered graph $\Gh$ to the true graph $G$.
This is because multiple graphs $G$ can give rise to the same PSNE set under different payoff functions
and thus be \emph{unidentifiable} from observations of joint actions alone (see section 4.4.1 of \cite{honorio15}
for a counter example.)
\section{Method}
\label{sec:method}
In this section, we describe our method for learning polymatrix games from observational data.
The individual and pairwise payoffs can be equivalently written, in linear form, as follows:
\begin{align*}
u^{i,i}(x_i) &= (\vtheta^{i,0})^T \vf^{i,0}(x_i),  \\
  u^{i,j}(x_i, x_j) &= (\vtheta^{i,j})^T \vf^{i,j}(x_i, x_j),
\end{align*}
where for $j \in \Nb_i$, $\vf^{i,j}(x_i, x_j) = (\Ind{x_i = a, x_j = b})_{a \in \A_i,\, b \in \A_j}$
and $\vtheta^{i,j}  = (\theta^{i,j}_{a,b})_{a \in A_i,\, b \in A_j}$,
$\vf^{i,0}(x_i) = (\Ind{x_i = a})_{a \in A_i} $ and $\vtheta^{i,0} = (\theta^{i,0}_{a})_{a \in A_i}$. 
Note that $\vf^{i,j} \in \Set{0,1}^{(m_i m_j)}$, $\vtheta^{i,j} \in \R^{(m_i m_j)}  \neq \vect{0}$,
$\vf^{i,0}(x_i) \in \Set{0,1}^{m_i}$, and $\vtheta^{i,0} \in \R^{m_i}$.
Let 
\begin{align}
\vtheta^i &\defeq \! (\vtheta^{i,0}, \vtheta^{i,1}, \ldots, \vtheta^{i,i-1}, \vtheta^{i,i+1}, \ldots ,\vtheta^{i,p}),  \notag \\
\vf^i(x_i, \vx_{\mi}) &\defeq \! (\vf^{i,0}(x_i), \vf^{i,1}(x_i, x_1), \ldots, \vf^{i,i-1}(x_i, x_{i-1}), \notag \\
	&\qquad \vf^{i,i+1}(x_i, x_{i+1}), \ldots, \vf^{i,p}(x_i, x_p)), \label{eq:groups}
\end{align}
with $\vtheta^{i,j} = \vect{0}$ for $j > 0 \Land j \notin \Nb_{i}$, and 
$\vtheta^i \in \R^{(m_i + \sum_{j \in \mi} m_i m_j)}, \vf^i(x_i, \vx_{\mi}) \in \Set{0,1}^{(m_i + \sum_{j \in \mi} m_i m_j)}$.
Thus the payoff for the $i$-th  player can be written, in linear form, as:
\begin{align}
u^i(x_i, \vx_{\mi}) = (\vtheta^i)^T \vf^i(x_i, \vx_{\mi}).
\end{align}
The learning problem then corresponds to learning the parameters $\vtheta^i$ for each player $i$. 
The sparsity pattern of $\vtheta^i$ identifies the neighbors of $i$. The way this differs from
the binary strategies considered by \cite{ghoshal2017learning} is that the parameters $\vtheta^i$
have a \emph{group-sparsity} structure, i.e., for all $j > 0 \Land j \notin \Nb_i$ the entire group 
of parameters $\vtheta^{i,j}$ is zero. In order to ensure that the payoffs are finite, we will assume
that the parameters for the $i$-th player belong to the set
$\Theta^i \defeq \Set{\vy \in \R^{(m_i + \sum_{j \in \mi} m_i m_j)} \,|\, \NormInfty{\vy} < \infty}$.

Our approach for estimating the parameters $\vtheta^i$ is to perform one-versus-rest multinomial logistic regression
with $\ell_{1,2}$ group-sparse regularization. In more detail, we obtain estimators $\vhtheta^i$ by solving the following
optimization problem for each $i \in [p]$:
\begin{align}
\vhtheta^i &= \argmin_{\vtheta \in \Theta^i} L^i(\Data; \vtheta) + \lambda \Norm{\vtheta}_{1,2}, \label{eq:optimization} \\
L^i(\Data; \vtheta) &= \frac{1}{n} \sum_{l=1}^n \loss^i(\vx^{(l)}; \vtheta), \\
	 \loss^i(\vx; \vtheta) &=  - \log 
	\left( \frac{\exp(\vtheta^T \vf^i(x_i, \vx_{\mi}))}{\sum_{a \in \A_i} \exp(\vtheta^T \vf^i(a, \vx_{\mi}))} \right ),
		\label{eq:loss}
\end{align}
where $\Norm{\vtheta}_{1,2} = \sum_{j \in [p]} \NormII{\vtheta_j}$, with $\vtheta_j$ being the $j$-th group of $\vtheta$.
When referring to a block of a matrix or vector we will use bold letters, e.g, $\vtheta_j$ denotes the $j$-th group or block of
$\vtheta$, while $\theta_j$ denotes the $j$-th element of $\vtheta$. 
In general, we define the $\ell_{a,b}$ group structured norm as follows: 
$\Norm{\vtheta}_{a, b} = \Norm{(\Norm{\vtheta_1}_b, \ldots, \Norm{\vtheta_p}_b)}_a$.
Also, when using group structured norms, we will use the group
structure as shown in \eqref{eq:groups}, i.e., we will assume that there are $p$ groups and, in the context of the $i$-th player,
the sizes of the groups are: $\Set{m_i, m_i m_1, \ldots, m_i m_{i-1}, m_i m_{i + 1}, \ldots, m_i m_p}$.
Finally, we will define the support set of $\vtheta^i$ as the set of all indices corresponding
to the active groups, i.e.,  $S_i = \Set{(j,k) | j \in \Set{0} \union \Nb_i \text{ and } k 
\in [m_i] \text{ for } j = 0, k \in [m_i m_j] \text{ for } j > 0}$, where $j$ can be thought of as indexing
the groups, while $k$ can be thought of as the indexing the elements within the $j$-th group.
Thus, $\Abs{S_i} = m_i + \sum_{j \in \Nb_i} m_i m_j$.

After estimating the parameters $\vhtheta^i$ for each $i \in [p]$, the payoff functions are simply estimated to be
$\uh^i(x_i, \vx_{\mi}) = (\vhtheta^i)^T \vf^i(x_i, \vx_{\mi})$. Finally, the graph $\Gh = ([p], \widehat{E})$ is 
given by the group-sparsity structure of $\uh^i$s, i.e., $\uh^{i,j}(\parg, \parg) \neq 0 \implies (i,j) \in \widehat{E}$.
\section{Sufficient Conditions}
First, we obtain sufficient conditions on the number of samples $n$ to ensure successful PSNE recovery. 
Since our theoretical results depend on certain properties of the Hessian of the loss function defined above, 
we introduce the Hessian matrix in this paragraph. Let $\mH^i(\vx; \vtheta)$ denote the Hessian of $\loss^i(\vx; \vtheta)$.
A little calculation shows that the $(j,k)$-th block of the Hessian matrix for the $i$-th player is given as:
\begin{align}
\mH^i_{j,k}(\vx; \vtheta) &= \sum_{\mcp{a \in \A_i}} \sigma^i(a, \vx_{\mi}; \vtheta) \vf^{i,j}(a, x_j) (\vf^{i,k}(a, x_k))^T - \notag \\
 &\;\; \Bigl\{ \Bigl(\sum_{\mcp{a \in \A_i}} \sigma^i(a, \vx_{\mi}; \vtheta) \vf^{i,j}(a, x_j)\Bigr) \times 
 \Bigl(\sum_{\mcp{a \in \A_i}} \sigma^i(a, \vx_{\mi}; \vtheta) \vf^{i,k}(a, x_k)\Bigr)^T \Bigr\}, \label{eq:hessian_block} \\
\sigma^i(x, \vx_{\mi}; \vtheta) &= \frac{\exp(\vtheta^T \vf^i(x, \vx_{\mi}))}{\sum_{a \in \A_i} \exp(\vtheta^T \vf^i(a, \vx_{\mi}))}, 
	\label{eq:sigma_i}
\end{align}
where we have overloaded the notation $\vf^{i,j}(x_i, x_j)$ to also include $\vf^{i,0}(x_i)$, i.e., we let
$\vf^{i,0}(x_i, x_0) \defeq \vf^{i,0}(x_i)$.
We will denote the $i$-th expected Hessian matrix at any parameter $\vtheta \in \Theta^i$
as $\mH^i(\vtheta) = \Exp{\vx}{\mH^i(\vx; \vtheta)}$,
and the $i$-th Hessian matrix at the true parameter $\vtheta^i$ as $\mH^i(\vtheta^i)$.
We will also drop the superscript $i$ from the $i$-th Hessian matrix, whenever clear from context.
We will denote the finite sample version of $\mH^i(\vtheta^i)$ by $\mH^i(\Data, \vtheta^i)$, i.e.,
$\mH^i(\Data, \vtheta^i) = \frac{1}{n}\sum_{l=1}^n \mH^i(\vx^{(l)}, \vtheta^i)$. Finally, we will
denote the Hessian matrix restricted to the true support set $S_i$ by: $\mH^i(\parg; \vtheta^i_{S_i}) \in \R^{\Abs{S_i} \times \Abs{S_i}}$.
In order to prove our main result, we will present a series of technical lemmas slowly building towards our main result.
Detailed proofs of the lemmas are given in Appendix \ref{app:proofs}.

The following lemma states that the $i$-th population Hessian is positive definite.
Specifically, the $i$-th population Hessian evaluated at the true parameter 
$\vtheta^i$, are positive definite with the minimum eigenvalue being $\cmin$.
We prove the following lemma by showing that the loss function given by \eqref{eq:loss},
when restricted to an arbitrary line, is strongly convex as long as the payoffs are finite.
\begin{lemma}[Minimum eigenvalue of population Hessian]
\label{lemma:min_eigenvalue}
For $\vtheta^i \in \Theta^i, \, \eigmin(\mH^i(\vtheta^i)) \defeq \cmin > 0$.
\end{lemma}
Given that population Hessian matrices are positive-definite, we then show that
the finite sample Hessian matrices, evaluated at any parameter 
$\vtheta_{S_i}$, are positive definite with high probability. We use tools 
from random matrix theory developed by \citep{Tropp2011} to prove the following lemma.
\begin{lemma}[Minimum eigenvalue of finite sample Hessian]
\label{lemma:min_eigenvalue_sample}
Let $\vtheta \in \Theta^i$ be any arbitrary vector and
let $\eigmin(\mH^i(\vtheta_{S_i})) \defeq \eigmin > 0$. Then, if the number of samples satisfies the following
condition:
\begin{align*}
n \geq \frac{8(d_i + 1)}{\eigmin} \log \left(\frac{m_i (1 + d_i m)}{\delta}\right),
\end{align*}
then $\eigmin(\mH^i(\Data; \vtheta_{S_i})) \geq \frac{\eigmin}{2}$ with probability at least $1 - \delta$
for some $\delta \in (0,1)$. 
\end{lemma}
Now that we have shown that the loss function given by \eqref{eq:loss} is strongly
convex (Lemmas \ref{lemma:min_eigenvalue} and \ref{lemma:min_eigenvalue_sample}), 
we exploit strong convexity to control the difference between the true parameter
and the estimator $\gNormIII{\vtheta^i - \vhtheta^i}$. 
However, before proceeding further, we need to bound the $\ell_{\infty, 2}$ norm of the gradient,
as done in the following lemma.
We prove the lemma by using McDiarmid's inequality to show that in each group
the finite sample gradient concentrates around the expected gradient, and then use a union bound
over all the groups to control the $\ell_{\infty, 2}$ norm.
\begin{lemma}[Gradient bound]
\label{lemma:gradient_bound}
Let $\Norm{\Exp{\vx}{\Grad \loss^i(\vx; \vtheta^i)}}_{\infty, 2} = \nu$, then
we have that 
\begin{align*}
\Norm{\Grad L^i(\Data; \vtheta^i)}_{\infty, 2} \leq \nu + \sqrt{\frac{2}{n} \log \Bigl( \frac{2(d_i + 1)}{\delta}} \Bigr),
\end{align*}
with probability at least $1 - \delta$.
\end{lemma}
Note that the expected gradient at the parameter $\vtheta^i$ does not vanish, i.e., 
$\Norm{\Exp{\vx}{\Grad \loss^i(\vx; \vtheta^i)}}_{\infty, 2} = \nu$. This is because of the mismatch
between the generating distribution $\Pf$ and the softmax distribution used for learning the parameters, as in \eqref{eq:loss}.
Indeed, if the data were drawn from a Markov random field, which induces a softmax distribution on
the conditional distribution of node given the rest of the nodes, the parameter $\nu = 0$. However this is not the case for us.
An unfortunate consequence of this is that, even with an infinite number of samples, our method will not 
be able to recover the parameters $\vtheta^i$ exactly. Thus, without additional assumptions on the payoffs,
our method only recovers the $\varepsilon$-Nash equilibrium of the game.

With the required technical results in place, we are now ready to bound $\gNormIII{\vtheta^i - \vhtheta^i}$.
Our analysis has two steps. First, we bound the norm in the true support set, i.e., $\gNormIII{\vtheta^i_{S_i} - \vhtheta^i_{S_i}}$.
Then, we show that the norm of the difference between the true parameter and the estimator, outside the support set,
is a constant factor (specifically 3) of the difference in the support set. For the first step
with use a proof technique originally developed by \cite{rothman_sparse_2008} in a different context,
while the second step follows from matrix algebra and optimality of the estimator $\vhtheta^i$ for the problem \eqref{eq:optimization}.

The following technical lemma, which will be used later on in our proof to bound $\gNormIII{\vhtheta^i_S - \vtheta^i_S}$,
lower bounds the minimum eigenvalue of the $i$-th population Hessian at an arbitrary parameter $\vtheta \in \Theta^i$,
in terms of the minimum eigenvalue of the $i$-th population Hessian at the true parameter $\vtheta^i$.
\begin{lemma}[Minimum population eigenvalue at arbitrary parameter]
\label{lemma:population_eigenvalue_arbit}
Let $\vtheta \in \Theta^i$ be any vector. Then the minimum eigenvalue of $i$-th population Hessian
matrix evaluated at $\vtheta_{S_i}$ is lower bounded as follows:
\begin{align*}
\eigmin(\mH^i(\vtheta_{S_i})) &\geq \eigmin(\mH^i(\vtheta^i_{S_i})) 
	- \frac{1}{4}(d_i + 1) m^2 \Norm{\vtheta_{S_i} - \vtheta^i_{S_i}}_{1,2}.
\end{align*}
\end{lemma}
Now, we are ready to bound the difference between the true parameter $\vtheta^i$
and its estimator $\vhtheta^i$, in the true support set $S_i$.
\begin{lemma}[Error of the $i$-th estimator on the support set]
\label{lemma:l2_bound}
If the regularization parameter and number of samples satisfy the following condition:
\begin{align*}
\lambda &\geq 2\left(\nu + \sqrt{\frac{2}{n} \log \left( \frac{2(d_i + 1)}{\delta} \right)}\right), \\
n &> \frac{2}{N(m, d_i)} \log \left(\frac{2(d_i + 1)}{\delta} \right),
\end{align*}
where $N(m, d_i) = \{\nicefrac{\cmin}{(36 m^2 (d_i + 1)^2)} - \nu\}^2$, and $\cmin \defeq \eigmin(\mH^i(\vtheta^i_{S_i}))$;
then with probability at least $1 - \delta$, for some $\delta \in (0,1)$, we have:
\begin{align}
\Norm{\vhtheta^i_{S_i} - \vtheta^i_{S_i}}_{1,2} \leq \frac{6 (d_i + 1)}{\cmin} \lambda.
\end{align}
\end{lemma}
Next, we bound the difference between the true parameter $\vtheta^i$
and its estimator $\vhtheta^i$.
\begin{lemma}[Error of the $i$-th parameter estimator]
\label{lemma:norm_bound}
Under the same conditions on the regularization parameter and number of samples 
as in Lemma \ref{lemma:l2_bound} we have, with probability at 
least $1 - \delta$ for some $\delta \in (0,1)$,
\begin{align*}
\gNormIII{\vhtheta^i - \vtheta^i} \leq \frac{24 (d_i + 1)}{\cmin} \lambda.
\end{align*}
\end{lemma}
Now that we have control over $\gNormIII{\vtheta^i - \vhtheta^i}$ for all $i \in [p]$, we are ready to prove
our main result concerning the sufficient number of samples needed by our method to guarantee
PSNE recovery with high probability.
\begin{theorem}
\label{thm:suff}
Let $\Gm = (G, \U)$, with $\U = \Set{u^i: \A_{\mi} \rightarrow \R}_{i \in [p]}$, be
the true potential graphical game over $p$ players and maximum degree $d$, from which the data set $\Data$ is drawn.
Let $\Gmh = (\widehat{G}, \widehat{\U})$, with $\widehat{U} = \Set{\uh^i:  \A_{\mi} \rightarrow \R}_{i \in [p]}$,
be the game learned from the data set $\Data$ by solving the optimization problem \eqref{eq:optimization} for each $i \in [p]$.
Then if the regularization parameter and the number of samples satisfy the condition:
\begin{align*}
\lambda &\geq 2\left(\nu + \sqrt{\frac{2}{n} \log \left( \frac{2p(d + 1)}{\delta} \right)}\right), \\
n &> \max \Biggl\{\frac{2}{N(m, d)} \log \left(\frac{2p(d + 1)}{\delta} \right),
	 \frac{8(d + 1)}{\cmin} \log \left(\frac{m (1 + d m)}{\delta}\right) \Biggr\},
\end{align*}
where $N(m, d) = \{\nicefrac{\cmin}{(36 m^2 (d + 1)^2)} - \nu\}^2$,
then we have that the following hold with probability at least $1 - \delta$, for some $\delta \in (0, 1)$:
\begin{enumerate}[(i)]
\item $\NE(\Gmh) = \eNE(\Gm)$, with $\varepsilon = \frac{48(d_i + 1)}{\cmin} \lambda$.
\item Additionally, if the true game $\Gm$ satisfies the condition:
$\forall i \in [p], \forall (x_i, \vx_{\mi}), (x_i', \vx_{\mi}) \in \A$ such that
$(x_i, \vx_{\mi}) \in \NE(\Gm) \Land (x_i', \vx_{\mi}) \notin \NE(\Gm) \implies u^i(x_i, \vx_{\mi}) > 
u^i(x_i', \vx_{\mi}) + \varepsilon$. Then, $\NE(\Gmh) = \NE(\Gm)$.
\end{enumerate}
\end{theorem}
\begin{proof}
Note that $\gNormInII{\vf^i(x_i, \vx_{\mi})} = 
\max \Set{\NormII{\vf^{i,0}(x_i)}, \NormII{\vf^{i,1}(x_i, x_1)}, \ldots, \NormII{\vf^{i,p}(x_i, x_p)}} = 1$,
for any $\vx \in \A$, since each binary vector $\vf^{i,j}(x_i, x_j)$ has a single ``1'' at exactly one location.
Then, from the Cauchy-Schwartz inequality, Lemma \ref{lemma:norm_bound}, and a union bound over all players, we have that:
\begin{align}
(\forall \vx \in \A,\, \forall i \in [p]) ~ \Abs{\uh^i(x_i, \vx_{\mi}) - u^i(x_i, \vx_{\mi})}  
		& = \Abs{(\vhtheta^i - \vtheta^i)^T \vf^i(x_i, \vx_{\mi})} \notag \\ 
		& \leq \gNormIII{\vhtheta^i - \vtheta^i} \gNormInII{\vf^i(x_i, \vx_{\mi})} \notag \\
		& = \gNormIII{\vhtheta^i - \vtheta^i} \leq \frac{24(d_i + 1)}{\cmin} \lambda = \frac{\varepsilon}{2},  \label{eq:payoff_diff}
\end{align}
with probability at least $1 - p\delta$. Now consider any $\vx \in \NE(\Gmh)$ and 
any $i \in [p]$. Since $\vx \in \NE(\Gmh)$, we have from \eqref{eq:payoff_diff}, $(\forall x_i' \in \A_i)$:
\begin{gather*}
u^i(x_i, \vx_{\mi}) + \nicefrac{\varepsilon}{2} \geq \uh^i(x_i, \vx_{\mi}) \geq \uh^i(x_i', \vx_{\mi})  \\
\implies u^i(x_i, \vx_{\mi}) \geq \uh^i(x_i', \vx_{\mi}) - \nicefrac{\varepsilon}{2}   \\
\implies u^i(x_i, \vx_{\mi}) \geq u^i(x'_i, \vx_{\mi}) - \varepsilon , 
\end{gather*}
where the last line again follows from \eqref{eq:payoff_diff}. This proves that $\NE(\Gmh) \subseteq \eNE(\Gm)$.
Using exactly the same arguments as above, we can also show that for any $\vx \in \NE(\Gm)$:
\begin{align*}
\uh^i(x_i, \vx_{\mi}) \geq \uh^i(x_i, \vx_{\mi}) - \varepsilon && (\forall x_i' \in \A_i),
\end{align*}
which proves that $\NE(\Gm) \subseteq \eNE(\Gmh)$. Thus we have that $\NE(\Gmh) = \eNE(\Gm)$, i.e., 
the set of joint strategy profiles
$\vx \in \NE(\Gmh)$ form an $\varepsilon$-Nash equilibrium set of the true game $\Gm$. This proves our first claim.
For our second claim, consider any $(x_i, \vx_{\mi}) \in \NE(\Gm)$ and $(x_i', \vx_{\mi}) \notin \NE(\Gm)$. Then:
\begin{gather*}
u^i(x_i, \vx_{\mi}) > u^i(x_i', \vx_{\mi}) + \varepsilon \\
\implies \uh^i(x_i, \vx_{\mi}) + \nicefrac{\varepsilon}{2} > \uh^i(x_i', \vx_{\mi}) - \nicefrac{\varepsilon}{2} + \varepsilon \\
\implies \uh^i(x_i, \vx_{\mi}) > \uh^i(x_i', \vx_{\mi}) \\
\implies (x_i, \vx_{\mi}) \in \NE(\Gmh) \Land (x_i', \vx_{\mi}) \notin \NE(\Gmh),
\end{gather*}
where the first line holds by assumption, and the
second line again follows from \eqref{eq:payoff_diff}. Thus we have that $\NE(\Gm) = \NE(\Gmh)$.
By setting the probability of error $p \delta = \delta'$ for some $\delta' \in (0, 1)$ we prove our claim.
The second part of the lower bound on the number of samples is due to Lemma \ref{lemma:min_eigenvalue_sample}.
\end{proof}
\begin{remark}
The sufficient number of samples needed by our method to guarantee PSNE recovery, with probability at
least $1 - \delta$,  scales as $\BigO{m^4 d^4 \log (\nicefrac{pd}{\delta})}$. This should be compared
with the results of \cite{jalali_learning_2011} for learning undirected graphical models. They show that
$\BigO{m^2 d^2 \log (m^2 p)}$ are sufficient for learning $m$-ary discrete graphical models. However,
their sample complexity hides a constant $K$ that is related to the maximum eigenvalue of the scatter
matrix, which we have upper bounded by $m^2 d^2$ in our case, leading to a slightly higher sample complexity.
\end{remark}

\begin{remark} Note that as $n \rightarrow \infty$, the regularization parameter $\lambda \rightarrow 2\nu$,
where $\nu$ is the maximum norm of the expected gradient at the true parameter $\vtheta^i$ across all $i \in [p]$.
Thus, even with an infinite number of samples, our method recovers the $\varepsilon$-Nash equlibria set of
the true game with $\varepsilon \rightarrow \frac{96(d_i + 1) \nu}{\cmin}$ as $n \rightarrow \infty$.
\end{remark}

\section{Necessary Conditions}
In this section, we obtain an information-theoretic lower bound on the
number of samples needed to learn sparse polymatrix games. Let $\sGm_{p,d,m}$
be set of polymatrix games over $p$ players, with degree at most $d$, and maximum
number of strategies per player being $m$. Our approach for doing so is to
treat the inference procedure as a communication channel, where nature picks
a game $\Gm^*$ from the set $\sGm_{p,d,m}$ and then generates a data set $\Data$ of $n$ strategy profiles.
A decoder $\psi: \A^n \rightarrow \sGm_{p,d,m}$ then maps $\Data$ to a game $\Gmh \in \sGm_{p,d,m}$.
We wish to obtain lower bounds on the number of samples required by any decoder $\psi$ to recover the true
game consistently. In this setting, we define the \emph{minimax} estimation error as follows:
\begin{align*}
\perr = \min_{\psi} \sup_{\Gm^* \in \sGm_{p,d,m}} \Prob{\NE(\psi(\Data)) \neq \NE(\Gm^*)},
\end{align*}
where the probability is computed over the data distribution. For obtaining necessary conditions
on the sample complexity, we assume that the data distribution follows
the global noise model described in \eqref{eq:global_noise}.
The following theorem prescribes the number of samples needed for learning sparse polymatrix games. 
Our proof of the theorem constitutes constructing restricted ensembles of ``hard-to-learn'' polymatrix games,
from which nature picks a game uniformly at random and generates data. We then use the Fano's technique 
to lower bound the minimax error. The use of restricted ensembles is customary for obtaining information-theoretic lower bounds,
c.f. \citep{santhanam2012information, Wang2010}.
\begin{figure*}[t]
\begin{center}
\fbox{
\begin{minipage}{0.97\linewidth}
\includegraphics[width=\linewidth]{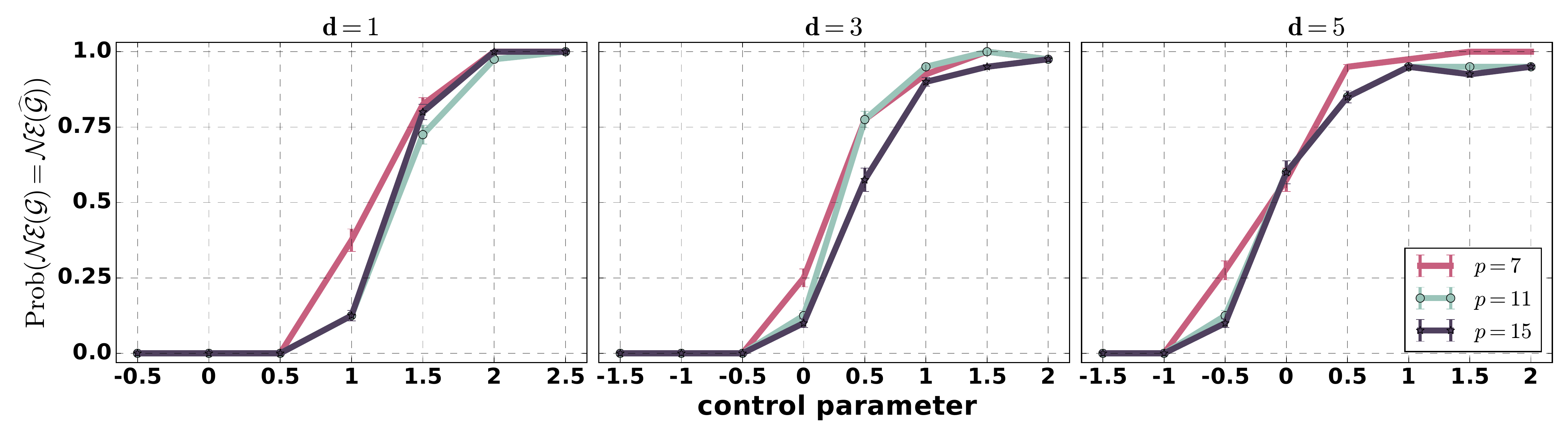}
\caption{Estimated probability of exact recovery of the PSNE set computed across
40 randomly sampled polymatrix games with the number of samples set to $n = 10^c (d+1)^2 \log (\nicefrac{2 p (d+1)}{\delta})$,
where $c$ is the control parameter shown in the x-axis, and $\delta = 0.01$. \label{fig:synthetic_res}}
\end{minipage}
}
\end{center}
\end{figure*}
\begin{theorem}
If the number of samples $n \leq \frac{\log (m^d - m){p \choose d}}{2 \log 2} - 1$,
then estimation fails with $\perr \geq \nicefrac{1}{2}$.
\end{theorem}
\begin{proof}
Consider the following restricted ensemble $\stGm \subset \sGm_{p,d,m}$ of $p$-player polymatrix games  
with degree $d$, and the set of pure-strategies of each player being $\A_i = [m]$.
Each $\Gm = (G, \U) \in \stGm_{p,d,m}$ is characterized by a set $\cI$ of \emph{influential players}, and
a set $\cI^c \defeq [p] \setminus \cI$ of \emph{non-influential} players, with $\Abs{\cI} = d$.
The graph $G$ is a complete (directed) bipartite graph from the set $\cI$ to $\cI^c$.
After picking the graph structure $G$, nature fixes the strategies of the influential players to some
$\va \in \Set{\vb \in [m]^{\Abs{\cI}} \mid \exists i,j \in \cI \text{ such that } b_i \neq b_j}$.
Finally, the payoff matrices are chosen as follows:
\begin{align*}
u^{i,i}(x_i) &= \Ind{x_i = a_i} && (\forall i \in \cI) \\
u^{j,j}(x_j) &= \frac{1}{(2x_j)} && (\forall j \in \cI^c) \\
u^{j,i}(x_j, x_i) &= \Ind{x_j = x_i} && (\forall i \in \cI \Land j \in \cI^c).
\end{align*}
Therefore, each $\Gm \in \stGm$ game has a exactly one unique Nash equilibrium where the
influential players play $\va$ (decided by nature) and the non-influential players play $\Maj(\va)$ ---
where $\Maj(\va)$ returns the majority strategy among $\va$, and in case of a tie between two or more 
strategies it returns the numerically lowest strategy (recall that the pure-strategy set for each player is $[m]$).
Thus we have that $\Abs{\stGm} = (m^d - m){p \choose d}$. 
Nature picks a game $\Gm$  uniformly at random from $\stGm$ by randomly selecting a set of $d$ players as ``influential'',
and then selecting a strategy profile $\va$ uniformly at random for the influential players and setting the payoff matrices
as described earlier. Nature then generates a dataset $\Data$
using the global noise model with parameter $q \in (\nicefrac{1}{m^p}, \nicefrac{2}{(m^p + 1)} ]$. 
Then from the Fano's inequality we have that:
\begin{align}
\perr \geq 1 -  \frac{\MI(\Data; \Gm) + \log 2}{H(\Gm)}, \label{eq:fano}
\end{align}
where $\MI(\parg; \parg)$ and $H(\parg)$ denote mutual information and entropy respectively.
The mutual information $\MI(\Data; \Gm)$ can be bounded, using a result by \cite{Yu97},
 as follows:
\begin{align}
\MI(\Data; \Gm) \leq \frac{1}{\Abs{\stGm}^2} \sum_{\Gm_1 \in \stGm} \sum_{\Gm_2 \in \stGm} \KL{\Pf_{\Data | \Gm = \Gm_1}}{\Pf_{\Data | \Gm = \Gm_2}},
	\label{eq:mi_bound}
\end{align}
where $\Pf_{\Data | \Gm = \Gm_1}$ (respectively $\Pf_{\Data | \Gm = \Gm_2}$) denotes the
data distribution under $\Gm_1$ (respectively $\Gm_2$). The KL divergence term from \ref{eq:mi_bound}
can be bounded as follows:
\begin{align}
\KL{\Pf_{\Data | \Gm = \Gm_1}}{\Pf_{\Data | \Gm = \Gm_2}} 
	&= n \sum_{\vx \in \A} \Pf_{\Data | \Gm = \Gm_1} \log \frac{\Pf_{\Data | \Gm = \Gm_1}}{\Pf_{\Data | \Gm = \Gm_2}} \notag \\
	&= n \Biggl\{\sum_{\vx \in \NE(\Gm_1)} q \log \frac{q (m^p - 1)}{1 - q} + 
		 \sum_{\vx \in \NE(\Gm_2)}  \frac{(1 - q)}{m^p - 1} \log \frac{1 - q}{q(m^p - 1)} \Biggr\} \notag \\
	&= \frac{n (qm^p -1)}{m^p -1} \log \left(\frac{q(m^p - 1)}{1 - q}	\right) \notag \\
	&\leq n \log \left(\frac{q(m^p - 1)}{1 - q}	\right) \leq n \log 2, \label{eq:kl_bound}
\end{align}
where the first line follows from the fact that the samples are \iid, the second line follows from
the fact the the distributions $\Pf_{\Data | \Gm = \Gm_1}$ and $\Pf_{\Data | \Gm = \Gm_2}$ assign
the same probability to $\vx \in \A \setminus (\NE(\Gm_1) \union \NE(\Gm_2))$, and the
last line follows from the fact that $q \in (\nicefrac{1}{m^p}, \nicefrac{2}{(m^p + 1)} ]$.
Putting together \eqref{eq:fano}, \eqref{eq:mi_bound} and \eqref{eq:kl_bound}, we have that
if
\begin{align*}
n \leq \frac{\log (m^d - m){p \choose d}}{2 \log 2} - 1,
\end{align*}
then $\perr \geq \nicefrac{1}{2}$. Since, learning the ensemble $\sGm$ is at least as hard 
as learning a subset of $\sGm$, our claim follows.
\end{proof}
\begin{remark}
From the above theorem we have that, the number of samples needed by any conceivable 
method, to recover the PSNE set consistently, is $\BigOm{d \log (pm)}$, assuming that $d = o(p)$.
Therefore, the method based on $\ell_{1,2}$-regularized logistic regression is information-theoretically 
optimal in the number of players, for learning sparse polymatrix games.
\end{remark}

\section{Experiments}
In order to validate our theoretical results, we performed various synthetic experiments by sampling a random polymatrix
game, generating data from the sampled game, and then using our method to learn the game from the sampled data. We estimated
the probability that our method learns the ``correct'' game, i.e., a game with the same PSNE set as the true game, across
40 randomly sampled games for each value of $p \in \Set{7, 11, 15}$ and $d \in \Set{1, 3, 5}$. 
The results are shown in Figure \ref{fig:synthetic_res}. We observe that the scaling of 
the sample complexity prescribed by Theorem \ref{thm:suff} indeed holds in practice. The results
show a phase transition behavior, where if the number of samples is less than
$c (d + 1)^2 \log (\nicefrac{p(d + 1)}{\delta})$, for some constant $c$, then PSNE recovery fails with high probability,
while if the number of samples is at least $C (d + 1)^2 \log (\nicefrac{p(d + 1)}{\delta})$, for some constant $C$,
then PSNE recovery succeeds with high probability. More details about our 
synthetic experiments can be found in Appendix \ref{sec:experiments}.

We also evaluated our algorithm on real-world data sets containing (i) U.S. supreme court rulings, 
(ii) U.S. congressional voting records, and (iii) U.N. General Assembly roll-call votes.  

Our algorithm recovers connected components 
corresponding to liberal and conservative blocs of justices within the Supreme Court of the U.S. The Nash equilibria 
consists of strategy profiles where all justices vote unanimously, as well as strategy profiles where the conservative and liberal
blocs vote unanimously but in opposition to each other. 

The game graph recovered from congressional voting records, 
groups Democrats and Republicans in separate components. Moreover, we observed that the connected components groups
senators belonging to the same state or geographic region together. The recovered PSNE set sheds light on the voting patterns
of senators --- senators belonging to the same connected component vote (almost) identically on bills. 

Finally, on the U.N. voting
data set our method recovered connected components comprising of Arab League countries and Southeast Asian countries respectively.
As was the case with the aforementioned data sets, the PSNE set grouped countries that vote almost identically on U.N. resolutions.

We were also able to compute the price of anarchy (PoA) for each data set, which quantifies the degradation of performance
caused by selfish behavior of non-cooperative agents. For the two supreme court voting data sets, the 
PoA was 1.9 and 1.6 respectively. For the congressional voting data set the PoA was 2.6, while for the united nations voting
data set the PoA was 3.0. More details and results from our real-world experiments can be found in Appendix \ref{sec:real_world}.

\paragraph{Concluding Remarks.}
We conclude this exposition with a discussion of potential avenues for future work. In this
paper we considered the problem of learning a very general, and widely used, class of graphical
games called polymatrix games, involving players with pure strategies. One can also consider 
\emph{mixed strategies}, which would entail learning distributions, instead of ``sets'', under
the framework of non-cooperative maximization of utility. Further, one can also consider
other solution concepts like \emph{correlated equilibria}.

\bibliographystyle{apalike}
\bibliography{paper}

\clearpage
\onecolumn
\begin{appendices}
\begin{center}
\textbf{Learning Sparse Polymatrix Games in Polynomial Time and Sample Complexity}
\end{center}

\section{Detailed Proofs}
\label{app:proofs}
\begin{proof}[Proof of Lemma \ref{lemma:min_eigenvalue} (Minimum eigenvalue of population Hessian)]
~\\
Fix any $\vtheta^0, \vtheta^1 \in \Theta^i$, with $\vtheta^1 \neq \vect{0}$.
For any $t \in (-\infty, \infty)$, 
let $F(t; x_i) \defeq (\vtheta^0 + t \vtheta^1)^T \vf(x_i, \vx_{\mi})$.
Then for $\vx \in \A$, 
\begin{align}
\loss(\vx; \vtheta^0 + t \vtheta^1) = - F(t; x_i) + \log(\sum_{a \in \A_i} \exp(F(t; a))). \label{eq:loss_line}
\end{align}
A little calculation shows that the double derivative of $\loss(\vx; \vtheta^0 + t \vtheta^1)$
with respect to $t$ is as follows:
\begin{align}
\ppdv{\loss(\vx; \vtheta^0 + t \vtheta^1)}{t} &= \sum_{\mathclap{a \in \A_i}} \sigma(t; a) F'(a)^2 
		- \Bigl(\sum_{\mathclap{a \in \A_i}} \sigma(t; a) F'(a)\Bigr)^2, \label{eq:ddl_derivative_t} \\
\sigma(t; b) &= \frac{\exp(F(t; b))}{\sum_{a \in \A_i} \exp(F(t; a))}, ~ (b \in \A_i)	\notag
\end{align}
where $F'(a)$ is the derivative of $F(t; a)$ with respect to $t$. Since $F(t; a)$ is a linear function of $t$, $F'(a)$
is not a function of $t$. Also note that $\sum_{a \in \A_i} \sigma(t; a) = 1$. Since $\vtheta^0, \vtheta^1$ have
bounded norm and $t \in (-\infty, \infty)$, we have that $\sigma(t; a) > 0, \forall a \in \A_i$. 
Therefore, from \eqref{eq:ddl_derivative_t}, the strict convexity of $(\parg)^2$ and Jensen's inequality, we have:
\begin{align*}
\ppdv{\loss(\vx; \vtheta^0 + t \vtheta^1)}{t} > 0 && (\forall t \in (-\infty, \infty)).
\end{align*}
Thus we have that $\loss(\vx, \vtheta)$ is strongly convex, i.e., $\eigmin(\mH^i(\vx; \vtheta)) > 0$,
$\forall \vtheta \in \Theta^i$. Finally, by concavity of $\eigmin(\parg)$ \citep{boyd2004convex}
and the Jensen's inequality we have:
\begin{align*}
\eigmin(\mH^i(\vtheta^i)) = \eigmin(\Exp{\vx}{\mH^i(\vx; \vtheta^i)}) \geq 
\Exp{\vx}{\eigmin(\mH^i(\vx; \vtheta^i))} > 0.
\end{align*}
\end{proof}
\begin{proof}[Proof of Lemma \ref{lemma:min_eigenvalue_sample} (Minimum eigenvalue of finite sample Hessian)]
~\\
To simply notation in the proof we will denote $S_i$ by $S$.
The $(j,k)$ block of $\mH(\Data; \vtheta_S)$, where $j, k \in \Set{0} \union \Nb_i$, can be written as:
\begin{align*}
&\mH_{j,k}(\Data; \vtheta_S) = \underbrace{\sum_{l=1}^n \sum_{\mathclap{a \in \A_i}} \sigma^i(a, \vx^{(l)}_{\mi}; \vtheta_S)
		 \vf^{i,j}(a, x^{(l)}_j) (\vf^{i,k}(a, x^{(l)}_k))^T}_{\mB_{j,k}(\Data; \vtheta_S)}
  - \underbrace{\sum_{l=1}^n \sum_{\;\;\;\mathclap{a, b \in \A_i}} \sigma^i(a, \vx^{(l)}_{\mi}; \vtheta_S) 
 		 \vf^{i,j}(a, x^{(l)}_j) \vf^{i,k}(b, x^{(l)}_k)^T}_{\mR_{j,k}(\Data; \vtheta_S)},
\end{align*}	 
where the matrices $\mB$ and $\mR$ have been defined above (blockwise). Since the matrix $\mR$
is positive semi-definite $\eigmax(\mH(\Data; \vtheta_S)) \leq \eigmax(\mB(\Data; \vtheta_S))$.
Further, since $\mB$ is positive semi-definite, we have, from Lemma \ref{lemma:max_eigenvalue_block_psd}:
\begin{align*}
\eigmax(\mB(\Data; \vtheta_S)) &\leq \sum_{j \in \Set{0} \union \Nb_i} \eigmax(\mB_{j,j}(\Data; \vtheta_S)) \\
	&\leq (d_i + 1) \max_{j \in \Set{0} \union \Nb_i} \eigmax\Bigl( \frac{1}{n}\sum_{l=1}^n \sum_{\mathclap{a \in \A_i}} 
		\sigma^i(a, \vx^{(l)}_{\mi}; \vtheta_S) \vf^{i,j}(a, x^{(l)}_j) (\vf^{i,j}(a, x^{(l)}_j))^T \Bigr) \\
	&\leq \frac{(d_i + 1)}{n} \max_{j \in \Set{0} \union \Nb_i} \sum_{l=1}^n \sum_{\mathclap{a \in \A_i}} 
		\sigma^i(a, \vx^{(l)}_{\mi}; \vtheta_S) \eigmax\left( \vf^{i,j}(a, x^{(l)}_j) (\vf^{i,j}(a, x^{(l)}_j))^T \right) \\
	&= d_i + 1.
\end{align*}
Thus we have that $\eigmax(\mH(\Data; \vtheta_S)) \leq \eigmax(\mB(\Data; \vtheta_S)) \leq d_i + 1 \defeq R$.
Also note that $\mH(\Data; \vtheta_S) \in \R^{\Abs{S} \times \Abs{S}}$, with $\Abs{S} \leq m_i (1 + d_i m)$.
Then using the matrix Chernoff bounds by \cite{Tropp2011}, we have:
\begin{align*}
\Prob{\eigmin(\mH(\Data; \vtheta_S)) \leq (1-\delta)\eigmin} \leq 
	\Abs{S} \left(\frac{\exp(-\delta)}{(1-\delta)^{(1 - \delta)}}\right)^{(\nicefrac{n \eigmin}{R})} 
\end{align*}
Setting $\delta = \nicefrac{1}{2}$ we get:
\begin{align*}
\Prob{\eigmin(\mH(\Data; \vtheta_S)) \geq \frac{\eigmin}{2}} \geq 1 - m_i (1 + d_i m) \exp\left(-\frac{n\eigmin}{8(d_i + 1)}\right)
\end{align*}
Controlling the probability of error to be at most $\delta$ we obtain the lower bound on the number of samples.
\end{proof}
\begin{proof}[Proof of Lemma \ref{lemma:gradient_bound} (Gradient bound)]
~\\
A simple calculation shows that 
\begin{align}
\pdv{\loss^i(\vx; \vtheta^i)}{\vtheta^{i,j}} = - \vf^{i,j}(x_i, x_j) + \sum_{a \in \A_i} \sigma^i(a, \vx_{\mi}; \vtheta^i) \vf^{i,j}(a, x_j), 
\end{align}
where $\sigma^i(\parg)$ has been defined in \eqref{eq:sigma_i}.
Let $\vg(\vx^{(1)}, \ldots, \vx^{(n)}) = (g_j(\vx^{(1)}, \ldots, \vx^{(n)}))_{j \in \Set{0} \union \Nb_i}$, where 
$g_j(\parg) = \NormII{\frac{1}{n} \sum_{l=1}^n \pdv{\loss^i(\vx^{(l)}; \vtheta^i)}{\vtheta^{i,j}}}$. 
Then $\NormInfty{\vg(\parg)} = \Norm{\Grad L^i(\Data; \vtheta^i)}_{\infty, 2}$ and
$\NormInfty{\Exp{\vx}{\vg(\parg)}} = \Norm{\Exp{\vx}{\Grad \loss^i(\vx; \vtheta^i)}}_{\infty, 2} = \nu$.
Then, for any $\vx^{(l)} \neq \vx^{(l)'}$ we have that:
\begin{align*}
&\Abs{g_j(\vx^{(1)}, \ldots, \vx^{(l)}, \ldots, \vx^{(n)}) - g_j(\vx^{(1)}, \ldots, \vx^{(l)'}, \ldots, g_j(\vx^{(n)})}  \\
& = \frac{1}{n} \aNormII{\vf^{i,j}(x^{(l)'}_i , x^{(l)'}_j) -\vf^{i,j}(x^{(l)}_i, x^{(l)}_j) 
					+ \sum_{\mcp{a \in \A_i}} \sigma^i(a, \vx^{(l)}_{\mi}; \vtheta^i) \vf^{i,j}(a, x^{(l)}_j) 
		             -  \sigma^i(a, \vx^{(l)'}_{\mi}; \vtheta^i) \vf^{i,j}(a, x^{(l)'}_j) } \\
	&\leq \frac{1}{n} \Bigl(2 + \sum_{a \in \A_i} (\sigma^i(a, \vx^{(l)}_{\mi}; \vtheta^i))^2 
			+ (\sigma^i(a, \vx^{(l)'}_{\mi}; \vtheta^i))^2 \Bigr)^{\nicefrac{1}{2}}
	 \leq \frac{1}{n} (2 + 2)^{\nicefrac{1}{2}} = \nicefrac{2}{n},
\end{align*}
where in the last line we used the fact that $\sum_{a} \sigma^i(a, \parg) = 1$ along with the Cauchy-Schwartz inequality.
Then using the McDiarmid's inequality we have:
\begin{align*}
\Prob{\Abs{g_j(\parg) - \Exp{\vx}{g_j(\parg)} } \leq t} \geq 1 - 2 \exp\Bigl( \frac{-nt^2}{2} \Bigr).
\end{align*}
Then using a union bound over all $j$ we have:
\begin{align*}
\Prob{\max_{j} \Abs{g_j(\parg) - \Exp{\vx}{g_j(\parg)} } \leq t} &\geq 1 - 2 (d_i + 1) \exp\Bigl( \frac{-nt^2}{2} \Bigr) \\
\implies  \Prob{\NormInfty{\vg(\parg) - \Exp{\vx}{\vg(\parg)}} \leq t} &\geq 1 - 2 (d_i + 1) \exp\Bigl( \frac{-nt^2}{2} \Bigr) \\
\implies \Prob{\NormInfty{\vg(\parg)} - \NormInfty{\Exp{\vx}{\vg(\parg)}} \leq t} &\geq 1 - 2 (d_i + 1) \exp\Bigl( \frac{-nt^2}{2} \Bigr) \\
\implies \Prob{\NormInfty{\vg(\parg)} \leq \nu + t} &\geq 1 - 2 (d_i + 1) \exp\Bigl( \frac{-nt^2}{2} \Bigr),
\end{align*}
where in the third line we used the reverse triangle inequality.
Setting the probability of error to be $\delta$ and solving for $t$, we prove our claim.
\end{proof}
\begin{proof}[Proof of Lemma \ref{lemma:population_eigenvalue_arbit} (Minimum population eigenvalue at arbitrary parameter)]
~\\
To simply notation in the proof we will denote $S_i$ by $S$.
The population Hessian matrix at $\mH(\vtheta_S)$ can also be written as $\mH(\vtheta^i_S + \DelS)$,
where $\DelS = \vtheta_S - \vtheta^i_S$. Using the variational characterization of the 
minimum eigenvalue of $\mH(\vtheta^i_S + \DelS)$ and the Taylor's theorem, we have:
\begin{align}
\eigmin(\mH(\vtheta^i_S + \DelS)) &= \min_{\Set{\vy \in \R^{\Abs{S}} | \NormII{\vy} = 1}}
	\sum_{i,j \in S} y_i \{ H_{i,j}(\vtheta^i_S) + (\Grad H_{i,j}(\vbtheta_S))^T \DelS \} y_j \notag \\
	&\geq \eigmin(\mH(\vtheta^i_S)) - \max_{\Set{\vy \in \R^{\Abs{S}} | \NormII{\vy} = 1}}
		\sum_{i,j \in S} y_i \{(\Grad H_{i,j}(\vbtheta_S))^T \DelS \} y_j \notag \\
	&\geq \eigmin(\mH(\vtheta^i_S)) - \max_{\Set{\vy \in \R^{\Abs{S}} | \NormII{\vy} = 1}} 
		\sum_{i,j \in S} y_i \{ \Abs{(\Grad H_{i,j}(\vbtheta_S))^T \DelS} \} y_j,	\label{eq:eigmin_arbit}	
\end{align} 
where $\vbtheta = t \vtheta^i_S + (1 - t) \vtheta_S$ for some $t \in [0,1]$, and
the third line follows from the monotonicity property of the spectral norm $\InNormII{\parg}$ \citep{johnson_monotonicity_1991}.
For any vector $\vtheta \in \Theta^i$, let $\mA(\vtheta_S) = (A_{i,j}(\vtheta_S))$, where
$A_{i,j}(\vtheta_S) =  \Abs{(\Grad H_{i,j}(\vtheta_S))^T \DelS}$.  Then,
\begin{align}
\InNormII{\mA(\vtheta_S)} = \InNormII{\Exp{\vx}{\mA(\vx; \vtheta_S)}} \leq \max_{\vx \in \A} \InNormII{\mA(\vx; \vtheta_S)}. \label{eq:matA}
\end{align}
Now consider the $(j,k)$ block of $\mA(\vx; \vtheta_S)$ for any $\vx \in \A$,
where $j, k \in \Set{0} \union \Nb_i$. Then, from \eqref{eq:hessian_block} we have that:
\begin{align*}
&\mA_{j,k}(\vx; \vtheta) = \underbrace{\sum_{a \in \A_i} \Abs{(\Grad \sigma^i(a, \vx_{\mi}; \vtheta))^T \DelS}
		 \vf^{i,j}(a, x_j) (\vf^{i,k}(a, x_k))^T}_{\mB_{j,k}(\vx; \vtheta)} \\
 &~ - \underbrace{\sum_{a, b \in \A_i} \Abs{\bigl\{\sigma^i(b, \vx_{\mi}; \vtheta) \Grad \sigma^i(a, \vx_{\mi}; \vtheta)   + 
 	\sigma^i(a, \vx_{\mi}; \vtheta) \Grad \sigma^i(b, \vx_{\mi}; \vtheta) \bigr\}^T\DelS} 
 		 \vf^{i,j}(a, x_j) \vf^{i,k}(a, x_k)^T}_{\mR_{j,k}(\vx; \vtheta)}.  
\end{align*}	 
Thus, $\mA(\vx; \vtheta) = \mB(\vx; \vtheta) - \mR(\vx; \vtheta)$, where the matrices $\mB$ and $\mR$ have been defined above (block-wise).
Observe that the matrix $\mR$ is positive semi-definite. Therefore, $\InNormII{\mA(\vx; \vtheta)} \leq \InNormII{\mB(\vx; \vtheta)}$.
Finally, since $\mB$ is positive semi-definite, the spectral norm of $\mB$ is at most the sum of the
spectral norms of the diagonal blocks (c.f. Lemma \ref{lemma:max_eigenvalue_block_psd}). Therefore, we have
\begin{align}
\InNormII{\mB(\vx; \vtheta)} \leq \sum_{j \in \Set{0} \union \Nb_i} \InNormII{\mB_{j,j}(\vx; \vtheta)}
\leq (d_i + 1) \Bigl( \max_{j \in \Set{0} \union \Nb_i} \InNormII{\mB_{j,j}(\vx; \vtheta)} \Bigr).\label{eq:matA1}
\end{align}
A little calculation shows that 
\begin{align*}
\pdv{\sigma^i(a, \vx_{\mi}; \vtheta)}{\vtheta_j} = \sigma^i(a, \vx_{\mi}; \vtheta)
\Bigl \{\vf^{i,j}(a, x_j) - \sum_{a' \in \A_i} \sigma^i(a', \vx_{\mi}; \vtheta) \vf^{i,j}(a', x_j) \Bigr\},
\end{align*}
and $\NormInfty{\ipdv{\sigma^i(a, \vx_{\mi}; \vtheta)}{\vtheta_j}} \leq \nicefrac{1}{4}$. 
Further, since for any given $a \in \A_i$, at most $m_j + 1$ elements of the partial derivative vector above 
is non-zero, we have $\NormII{\ipdv{\sigma^i(a, \vx_{\mi}; \vtheta)}{\vtheta_j}} \leq \nicefrac{(m_j + 1)}{4}$
and $\Norm{\Grad \sigma^i(a,\vx_{\mi}; \vtheta)}_{\infty, 2} \leq \nicefrac{(m_j + 1)}{4} \leq \nicefrac{(m + 1)}{4}$.
Then using the Cauchy-Schwartz inequality and the monotonicity property of spectral norm \citep{johnson_monotonicity_1991}
 we have:
\begin{align}
\InNormII{\mB_{j,j}(\vx; \vtheta)} &\leq 
	\InNormII{\sum_{a \in \A_i} \Norm{\Grad \sigma^i(a, \vx_{\mi}; \vtheta))}_{\infty, 2} \Norm{\DelS}_{1,2}
		 \vf^{i,j}(a, x_j) (\vf^{i,j}(a, x_j))^T} \notag \\
	&\leq \frac{1}{4} m_i m \Norm{\DelS}_{1,2} 		 \label{eq:matB}
\end{align}
Putting together \eqref{eq:eigmin_arbit}, \eqref{eq:matA}, \eqref{eq:matA1} and \eqref{eq:matB} we get
\begin{align*}
\eigmin(\mH(\vtheta^i_S + \DelS)) &\geq \eigmin(\mH(\vtheta^i_S)) - \InNormII{\mA(\vtheta_S)} \\
&\geq \eigmin(\mH(\vtheta^i_S)) -  (d_i + 1) \Bigl(\max_{\vx \in \A} \max_{j \in \Set{0} \union \Nb_i} 
		\InNormII{\mB_{j,j}(\vx; \vtheta)} \Bigr)\\
&\geq \eigmin(\mH(\vtheta^i_S)) - \frac{1}{4}(d_i + 1) m_i m \Norm{\DelS}_{1,2}.
\end{align*}
\end{proof}
\begin{proof}[Proof of Lemma \ref{lemma:l2_bound} (Error of the i-th estimator on the support set)]
~\\
To simplify notation in the proof, we will write $S$ instead of $S_i$. 
Recall that $L^i(\Data; \vtheta)$ is the empirical loss for the $i$-th player for parameter $\vtheta$.
For the purpose of the proof we will often write $L(\vtheta)$ instead of $L^i(\Data; \vtheta)$.
Let $F(\vtheta) = L(\vtheta) + \lambda \Norm{\vtheta}_{1,2}$. For any $\vtheta \in \Theta^i$,
let $\DelS = \vtheta_S - \vtheta^i_S$ denote the difference between $\vtheta$ and the true parameter $\vtheta^i$
on the true support set $S$. We introduce the following shifted and reparameterized regularized loss function:
\begin{align}
\Ft(\DelS) = \underbrace{L(\vtheta^i_S + \DelS) - L(\vtheta^i_S)}_{\text{term 1}} +
	 \underbrace{\lambda (\Norm{\vtheta_S^i + \DelS}_{1,2} - \Norm{\vtheta^i_S}_{1,2})}_{\text{term 2}}, \label{eq:f_tilde}
\end{align}
which takes the value $0$ at the true parameter $\vtheta^i$, i.e., $\Ft(\vect{0}) = 0$. 
Let $\DelhS = \vhtheta^i_S - \vtheta^i_S$, where $\vhtheta^i$ minimizes $F(\vtheta)$.
Since $\vhtheta^i$ minimizes $F(\vtheta)$, we must have that $\Ft(\DelhS) \leq 0$. Thus, in order to upper bound
$\Norm{\DelhS}_{1,2} = \Norm{\vhtheta_S - \vtheta^i_S}_{1,2} \leq b$, we show that there exists an $\ell_{1,2}$
ball of radius $b$ such that function $\Ft(\DelS)$ is strictly positive on 
the surface of the ball. To see this, assume the contrary, i.e,
$\forall \vDel \in \Theta^i \Land \Norm{\DelS}_{1,2} = b,\, \Ft(\DelS) > 0$,
but $\DelhS$ lies outside the ball, i.e., $\Norm{\DelhS}_{1,2} > b$. Then, there 
exists a $t \in (0, 1)$ such that $(1 - t) \vect{0} + t \DelhS$ lies on the surface of the ball, i.e.,
$\Norm{(1 - t) \vect{0} + t \DelhS}_{1,2} = b$. However, by convexity of $\Ft$ we have that
\begin{align*}
 	0 < \Ft((1 - t) \vect{0} + t \DelhS) \leq (1 - t) \Ft(\vect{0}) + t \Ft(\DelhS) = t \Ft(\DelhS),
\end{align*}
which implies that $\Ft(\DelhS) > 0$ and therefore is a contradiction to the fact that $\Ft(\DelhS) \leq 0$. 
Going forward, our strategy would be to lower bound $\Ft(\DelS)$ in terms of $\Norm{\DelS}_{1,2} = b$.
We then set the lower bound to $0$ and solve for $b$, to obtain the radius of the $\ell_{1,2}$ ball
on which the function is non-negative. Towards that end we first lower bound the first term of \eqref{eq:f_tilde}.

Using the Taylor's theorem and the Cauchy-Schwartz inequality, for some $t \in [0, 1]$, we have:
\begin{align}
&L(\vtheta^i_S + \DelS) - L(\vtheta^i_S) \notag \\
&\quad  = \Grad L(\vtheta^i_S)^T \DelS + \DelS^T \Grad^2L(\vtheta^i_S + t \DelS) \DelS, \notag \\
&\quad \geq - \Norm{\Grad L(\vtheta^i_S)}_{\infty, 2} \Norm{\DelS}_{1,2}  
	+ \NormII{\DelS}^2 \eigmin(\mH(\Data; \vtheta^i_S + t \DelS)) \notag  \\
&\quad \geq - \frac{b\lambda}{2} 
	+ \frac{\Norm{\DelS}^2_{1, 2}}{d_i + 1} \eigmin(\mH(\Data; \vtheta^i_S + t \DelS)) \notag\\
&\quad \geq - \frac{b\lambda}{2}  +
	\frac{b^2}{2(d_i + 1)} \left( \cmin - \frac{m^2 b (d_i + 1)}{4} \right)  \notag \\
&\quad \geq - \frac{b \lambda}{2}  
	+	\frac{b^2 \cmin}{4 (d_i + 1)}, \label{eq:term1_lb}
\end{align}
where the third follows from our assumption that $\gNormInII{\Grad L(\vtheta^i)} \leq \nicefrac{\lambda}{2}$
and the fact for any vector $\vx$, $\NormII{\vx} \geq (\nicefrac{1}{\sqrt{g}}) \Norm{\vx}_{1,2}$ 
where the $\ell_{1,2}$ norm is evaluated over $g$ groups.
The fourth line follows from Lemma \ref{lemma:population_eigenvalue_arbit} with $t = 1$ and Lemma \ref{lemma:min_eigenvalue_sample}.
Finally, in the last line we assumed that $b \leq \nicefrac{2\cmin}{(m^2 (d_i + 1))}$ --- 
an assumption that we will verify momentarily.
The second term of $\eqref{eq:f_tilde}$ is easily lower bounded using the reverse triangle inequality as follows:
\begin{align}
\lambda (\Norm{\vtheta_S^i + \DelS}_{1,2} - \Norm{\vtheta^i_S}_{1,2})
& \geq - \lambda \Norm{\DelS}_{1,2} = - b \lambda \label{eq:term2_lb}
\end{align}
Putting together \eqref{eq:f_tilde}, \eqref{eq:term1_lb} and \eqref{eq:term2_lb} we get:
\begin{align*}
\Ft(\DelS) \geq -\frac{b \lambda}{2} + \frac{b^2 \cmin}{4(d_i + 1)} - b \lambda.
\end{align*}
Setting the above to zero and solving for $b$ we get:
\begin{align*}
b = \frac{6\lambda(d_i + 1)}{\cmin}.
\end{align*}
Finally, coming back to our assumption that $b \leq \nicefrac{2\cmin}{(m^2 (d_i + 1))}$,
it is easy to show that the assumption holds if the regularization parameter $\lambda$ satisfies:
\begin{align*}
\lambda \leq \frac{\cmin^2}{3m^2(d_i + 1)^2},
\end{align*}
The lower bound on the number of samples is obtained by ensuring that the lower bound on 
$\lambda$ is less than the upper bound.
The final claim follows from using the high probability bound on $\gNormInII{\Grad L(\vtheta^i)}$ from
Lemma \ref{lemma:gradient_bound}. 
\end{proof}
\begin{proof}[Proof of Lemma \ref{lemma:norm_bound} (Error of the i-th parameter estimator)]
~\\
$\vDel \defeq \vhtheta_i - \vtheta^i$. We will denote the true support of $\vtheta^i$ by $S$, 
and the complement of $S$ by $S^c$. We will also simply write $L(\vtheta)$ instead of $L^i(\Data; \vtheta)$.
For any vector $\vy$, let $\vy_{\Sb}$ denote the vector $\vy$ with elements not in the support
set $S$ zeroed out, i.e., 
\begin{align*}
(\vy_{\Sb})_j =\left\{ \begin{array}{cl}
y_j & j \in S, \\
0 & \text{otherwise}
\end{array}\right.
\end{align*}
Then by definition of $S$, we have that $\Norm{\vtheta^i_{\Sb}}_{1,2} = \Norm{\vtheta^i}_{1,2}$.
\begin{align*}
\gNormIII{\vhtheta^i} &= \gNormIII{\vtheta^i + \vDel} = \gNormIII{\vtheta^i_{\Sb} + \DelSb + \DelSbc} \\
&= \gNormIII{\vtheta^i_{\Sb} + \DelSb} + \gNormIII{\DelSbc} \\
&\geq \gNormIII{\vtheta^i_{\Sb}} - \gNormIII{\DelSb} + \gNormIII{\DelSbc},
\end{align*}
where in the second line follows from the fact that the index sets $S$ and $S^c$ have non-overlapping groups,
and in the last line we used the reverse triangle inequality. Rearranging the terms of the previous equation,
and from the fact that $\Norm{\vtheta^i_{\Sb}}_{1,2} = \Norm{\vtheta^i}_{1,2}$, we get:
\begin{align}
\gNormIII{\vtheta^i} - \gNormIII{\vhtheta^i} \leq \gNormIII{\DelSb} - \gNormIII{\DelSbc} \label{eq:l1norm_bound1}
\end{align}

Next, by optimality of $\vhtheta^i$ we have that 
$L(\vtheta^i) + \lambda \gNormIII{\vtheta^i} \geq L(\vhtheta^i) + \lambda \gNormIII{\vhtheta^i}$. 
Rearranging the terms and continuing, we get 
\begin{align}
\lambda(\gNormIII{\vtheta^i} - \gNormIII{\vhtheta^i})  &\geq L(\vhtheta^i) - L(\vtheta^i) \notag \\
 &\geq (\Grad L(\vhtheta^i))^T (\vhtheta^i - \vtheta^i) \notag \\
 &\geq - \gNormInII{\Grad L(\vhtheta^i))^T} \gNormIII{\vDel} \notag \\
  &\geq - \frac{\lambda}{2} \gNormIII{\vDel}, \label{eq:l1norm_bound2}
\end{align}
where the third line follows from the convexity of $L(\parg)$, the fourth line follows from 
the Cauchy-Schwartz inequality and the last line follows from our assumption that 
$\lambda \geq 2 \gNormInII{\Grad L(\vtheta^i)}$. Thus, from \eqref{eq:l1norm_bound1} and
\eqref{eq:l1norm_bound2} we have that
\begin{align*}
\frac{1}{2} \gNormIII{\vDel} &\geq \gNormIII{\DelSbc} - \gNormIII{\DelSb} \\
\implies \frac{1}{2} \gNormIII{\DelSb} + \frac{1}{2} \gNormIII{\DelSbc} &\geq \gNormIII{\DelSbc} - \gNormIII{\DelSb} \\
\implies 3 \gNormIII{\DelSb} &\geq \gNormIII{\DelSbc}.
\end{align*}
Finally, from the above inequality, we have $\gNormIII{\vDel} = \gNormIII{\DelSb} + \gNormIII{\DelSbc} \leq 4 \gNormIII{\DelS}$.
The final result follows from the upper bound on $\gNormIII{\DelS}$ derived in Lemma \ref{lemma:l2_bound}.
\end{proof}
\begin{lemma}[Maximum eigenvalue of block positive semi-definite matrix]
\label{lemma:max_eigenvalue_block_psd}
Let $\mX \in \R^{n \times n} \succeq \vect{0}$ be any positive semi-definite matrix,
with $\mX_{i,i}$ being the $i$-th diagonal block of $\mX$. Then
\begin{align*}
\eigmax(\mX) \leq \sum_{i} \eigmax(\mX_{i,i})
\end{align*}
\end{lemma}
\begin{proof}
We will prove the result by decomposing $\mX$ into two blocks as follows:
\begin{align*}
\mX = \matrx{\mX_{1,1} & \mX_{1,2} \\ \mX_{2,1} & \mX_{2,2}} ,
\end{align*}
where $\mX_{1,1} \in \R^{n_1 \times n_1}$, $\mX_{2,2} \in \R^{n_2 \times n_2}$ and $n_1 + n_2 = 1$.
The general result for multiple diagonal blocks is obtained by recursively 
decomposing the blocks $\mX_{1,1}$ and $\mX_{2,2}$.
Any unit vector $\vx$ can be written as $\vx = c_1(\vx) \vx_1(\vx) + c_2(\vx) \vx_2(\vx)$,
with $\vx_1(\vx) = (\nicefrac{x_1}{\NormII{\vx_1(\vx)}}, \ldots, \nicefrac{x_{n_1}}{\NormII{\vx_1(\vx)}}, \vect{0})$,
 $\vx_2(\vx) = (\vect{0}, \nicefrac{x_{n_2}}{\NormII{\vx_2(\vx)}}, \ldots, \nicefrac{x_{n}}{\NormII{\vx_2(\vx)}})$,
and $c_1(\vx) = \NormII{\vx_1(\vx)}$ (similarly $c_2(\vx)$). For notational simplicity we will drop the $(\vx)s$.
Note that $c_1^2 + c_2^2 = 1$, thus $\vc = (c_1, c_2)$ is also a unit vector. Further, for any
unit vector $\vx$, we have $\vx^T \mX \vx = \vc^T \mY \vc$, where
\begin{align*}
        \mY = \matrx{\vx_1^T \mX \vx_1 & \vx_1^T \mX \vx_2 \\
                \vx_2^T \mX \vx_1 & \vx_2^T \mX \vx_2} \in \R^{2 \times 2}.
\end{align*}
Note that $\vx_1^T \mX \vx_1 \leq \eigmax(\mX_{1,1})$ and $\vx_2^T \mX \vx_2 \leq \eigmax(\mX_{2,2})$ for all $\vx$.
Thus, using the variational characterization of the maximum eigenvalue of $\mX$ we get:
\begin{align*}
\eigmax(\mX) &= \max_{\NormII{\vx} = 1} \vx^T \mX \vx \\
        &= \max_{\Set{\vc = (\NormII{\vx_1(\vx)}, \NormII{\vx_2(\vx)}) \,:\, \NormII{\vx} = 1}} \vc^T \mY \vc \\
        &\leq \max_{\NormII{\vc} = 1} \vc^T \mY \vc = \eigmax(\mY) \leq \Tr{\mY} && \text{(since $\mY$ is positive semi-definite)}\\
        &\leq  \eigmax(\mX_{1,1}) + \eigmax(\mX_{2,2}),
\end{align*}
where the third line follows from the fact that the maximization is over a superset of the set 
$\Set{\vc = (\NormII{\vx_1(\vx)}, \NormII{\vx_2(\vx)}) \,:\, \NormII{\vx} = 1}$.
\end{proof}
\section{Details of Synthetic Experiments}
\label{sec:experiments}
We generated random polymatrix games $\Gm$ by first generating random graphs over $p$ players with degree exactly $d$,
and number of pure strategies $m = 3$ per player. For each edge $(i,j)$ in the graph, we set the payoffs as follows:
\begin{align*}
u^{i,i}(a) &= 0 && (\forall a \in [3]) \\
u^{i,j}(a,b) &\sim \mathcal{N}(0, 2) && (\forall a \in [2] \Land b \in [3]) \\
u^{i,j}(3,b) &= 0 && (\forall b \in [3])
\end{align*}
We then generated a data set $\Data$ from the
game using the local noise model \eqref{eq:local_noise}, with the noise parameter $q_i = 0.6$ for all $i \in [p]$. We then used our
method to learn a game $\Gmh$ from the data set $\Data$, and computed $\Ind{\NE(\Gmh) = \NE(\Gm)}$. 
We then estimated the probability of successful PSNE recovery, $\Prob{\NE(\Gmh) = \NE(\Gm)}$, across 40 randomly 
sampled polymatrix games. Figure \ref{fig:synthetic_res} plots the probability of successful PSNE recovery
as the number of samples is varied as $n = 10^c (d+1)^2 \log (\nicefrac{2 p (d+1)}{\delta})$ and for various values
of $d \in \Set{1, 3 , 5}$, with $c$ being the control parameter and $\delta = 0.01$. 
\section{Experiments on real-world data}
\label{sec:real_world}
We validated our method on three publicly available real-world data sets containing 
(a) U.S. supreme court justices rulings, (b) voting records of senators from 
the 114th U.S. congress, and (c) roll-call votes in the U.N. General Assembly. 
We present evaluations of our method for each of the data set below.

\subsection{Supreme court voting records}
We analyzed two data sets of supreme court rulings: the first data set contains rulings of 
9 justices across 512 cases spanning years 2010 to 2014,
while the second data set contains rulings of 8 justices across 75 cases from year 2015 onwards 
\footnote{All the data sets are publicly available at \url{http://scdb.wustl.edu}.}.
We pre-processed the data, according the available code book, to map the vote of each justice, which was originally an integer between 1 to 8,
to an integer between 1 to 3. Votes $\Set{1,3,4,5}$ were mapped to $1$ and was interpreted as ``voting with majority'',
votes $\Set{6,7,8}$ were mapped to $2$ and was interpreted as ``not participating in the decision'' , 
while vote $3$ was mapped to $2$ and was interpreted as ``dissent''. 
Thus, after pre-processing, each justice's vote was an integer between 1 to 3,
with 1 corresponding to majority, 2 corresponding to abstention, and 3 corresponding to dissent.

After pre-processing the data, we learned a polymatrix game over supreme court justices using our algorithm.
The regularization parameter $\lambda$ was set according to Theorem \ref{thm:suff} with reasonable values for different
unknown population parameters. A more principled way to chose the regularization parameter $\lambda$ is to assume a specific
observation model, for instance, the global or local noise model, and then using crossvalidation to maximize the log-likelihood. 
The game graphs are shown in Figure \ref{fig:sc_graph} and the PSNE sets are shown in Table \ref{tab:sc_psne}
for the two supreme court rulings data sets (years 2010-2014 and year 2015 onwards).

From \ref{fig:sc_graph} it is clear that our method recovers the well-established ideologies 
of the supreme court justices. This is especially evident for the graph learned from the first data set --- there are two strongly
connected components corresponding to the conservative and liberal bloc within the supreme court. The PSNE set recovered by our
algorithm is also quite revealing. In both the data sets, a unanimous vote of 1 is a Nash equilibrium. Justice Kennedy, who
has a moderate jurisprudence, always votes with the majority in the PSNE set. Further, strategy profiles
where the conservative blocs and liberal blocs vote unanimously but dissent against each other are also in the PSNE set.
In the second data set, there is a strongly connected component between the justice 
Kagan, Kennedy, and Breyer --- this also bears out in the corresponding PSNE set 
where the strategies of the three justices are identical.

To compute the price of anarchy (PoA), we shifted all the payoff matrices by a constant to make the payoffs non-negative. Note that
this does not change the PSNE set of the game. The price of anarchy was computed to be the ratio between the maximum welfare
across all strategy profiles and the minimum welfare across all strategy profiles in the PSNE set. The PoA for the two data
sets were, respectively, 1.9104 and 1.6115.
\begin{figure}[htbp]
\centering
\includegraphics[width=0.3\linewidth]{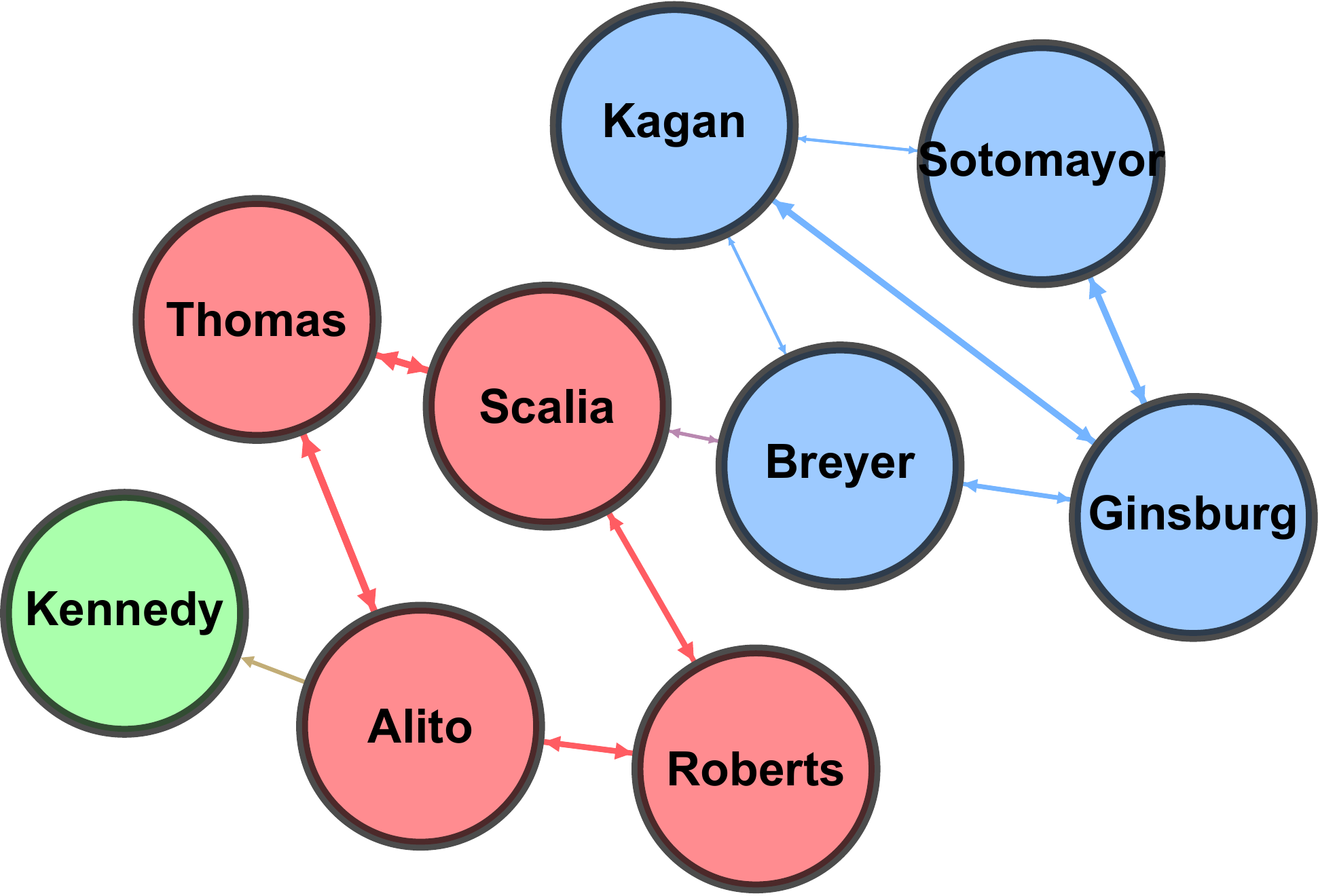}\hspace*{2cm}%
\includegraphics[width=0.24\linewidth]{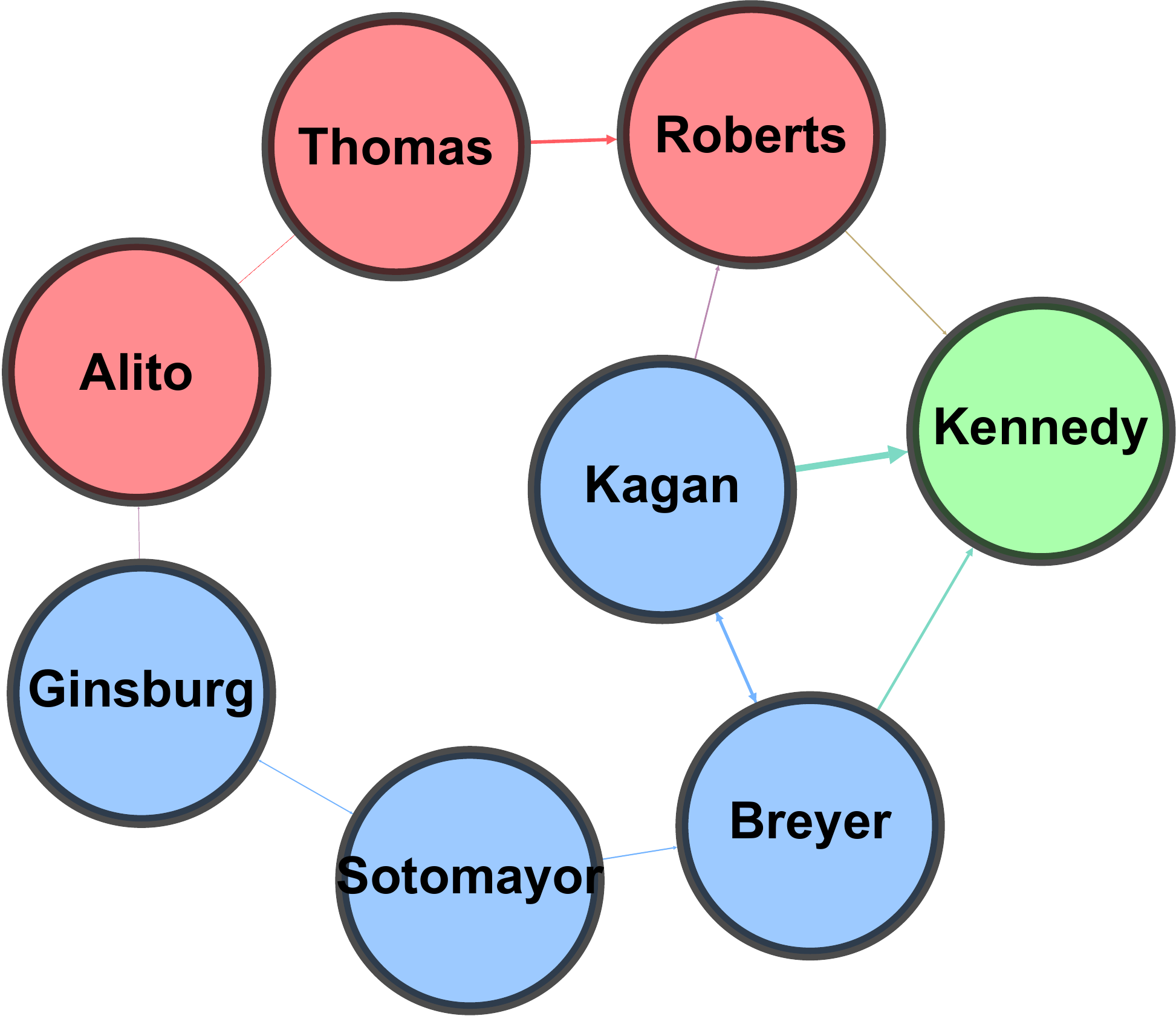}
\caption{The graphical game recovered from supreme court rulings data set 1 (years 2010-2014)
on the left, and data set 2 (year 2015 onwards) on the right. Justice Thomas, Scalia, Roberts and Alito are widely
known to be conservative and are denoted by the color \Cons{~}, while Justice Breyer, Kagan, Sotomayor and
Ginsburg, who are known to have a more liberal jurisprudence, are denoted by color \Lib{~}. Justice Kennedy,
who has a reputation of being moderate, is denoted by the color \Neu{~}.
The game graph was generated by adding all edges $(i, j)$ 
if the corresponding payoff matrix $u^{i,j}$ was not all zeros. The average ``influence'' from $j$ to $i$
was calculated as the mean absolute payoff, i.e., $\frac{1}{6}\sum_{a=2}^{3} \sum_{b=1}^3 \Abs{u^{i,j}(a, b)}$.
The thickness of the edge denotes this influence of player $j$ on $i$. Only the top $50\%$
of the edges, in terms of influence, are shown.
 \label{fig:sc_graph}}
\end{figure}
\begin{table}[htbp]
\centering
\setlength{\tabcolsep}{0pt}
\resizebox{0.8\textwidth}{!}{%
\begin{tabular}{ccccccccc}
\hline
\Cons{Thomas} & \Cons{Scalia} & \Cons{Alito} & \Cons{Roberts}
	& \Neu{Kennedy} & \Lib{Breyer} & \Lib{Kagan} & \Lib{Ginsburg} & \Lib{Sotomayor} \\
\hline
     1  &   1  &   1  &   1  &   1  &   1  &   1  &   1  &   1 \\
     1  &   1  &   1  &   1  &   1  &   3  &   3  &   3  &   3 \\
     2  &   2  &   2  &   2  &   1  &   2  &   2  &   2  &   2 \\
     3  &   3  &   3  &   3  &   1  &   1  &   1  &   1  &   1 \\
     3  &   3  &   3  &   3  &   1  &   3  &   3  &   3  &   3 \\
\hline     
\end{tabular}
}
\resizebox{0.8\textwidth}{!}{%
\begin{tabular}{ccccccccc}
\hline
\Cons{Thomas} & \Cons{Alito} & \Cons{Roberts}
	& \Neu{Kennedy} & \Lib{Breyer} & \Lib{Kagan} & \Lib{Ginsburg} & \Lib{Sotomayor} \\
\hline
     1  &   1  &   1  &   1  &   1  &   1  &   1  &   1 \\
     2  &   2  &   2  &   2  &   2  &   2  &   2  &   2 \\
     2  &   2  &   2  &   2  &   2  &   2  &   3  &   3 \\
     3  &   3  &   3  &   2  &   2  &   2  &   3  &   3 \\
\hline     
\end{tabular}
}
\caption{The PSNE set learned from supreme court rulings data sets 1704 (top) and 1705 (bottom) respectively. 
Colors represent \Cons{conservative}, \Lib{liberal}, and \Neu{neutral} justices respectively. The price of
anarchy for the two data sets was computed to be 1.9 and 1.6 respectively.
\label{tab:sc_psne}}
\end{table}

\subsection{Senate voting records}
We analyzed U.S. congressional voting records for the second session of the 114th congress (January 4, 2016 to January 3, 2017)
\footnote{The data set is publicly available at \url{http://www.senate.gov/legislative/votes.htm}}. The data set comprised of
the votes of 100 senators on 63 bills.
The votes were pre-processed to take one of the three values: 1 (``yes''), 2 (``abstention''), and 3 (``no'').
After pre-processing the data set we ran our algorithm to recover a polymatrix game from congressional voting records. 
Figure \ref{fig:sen_graph} shows the recovered game graph. Once again our method recovers the connected components 
corresponding the republicans and democrats. Interestingly, the connected components also have a nice geographic interpretation,
for instance, the graph groups senators from Idaho, New Mexico, New York and midwestern states in their respective connected components.
Strategy profiles where the overwhelming majority of senators in a connected component vote ``yes'' are in equilibria.

\begin{figure}[htbp]
\centering
\includegraphics[width=0.4\linewidth]{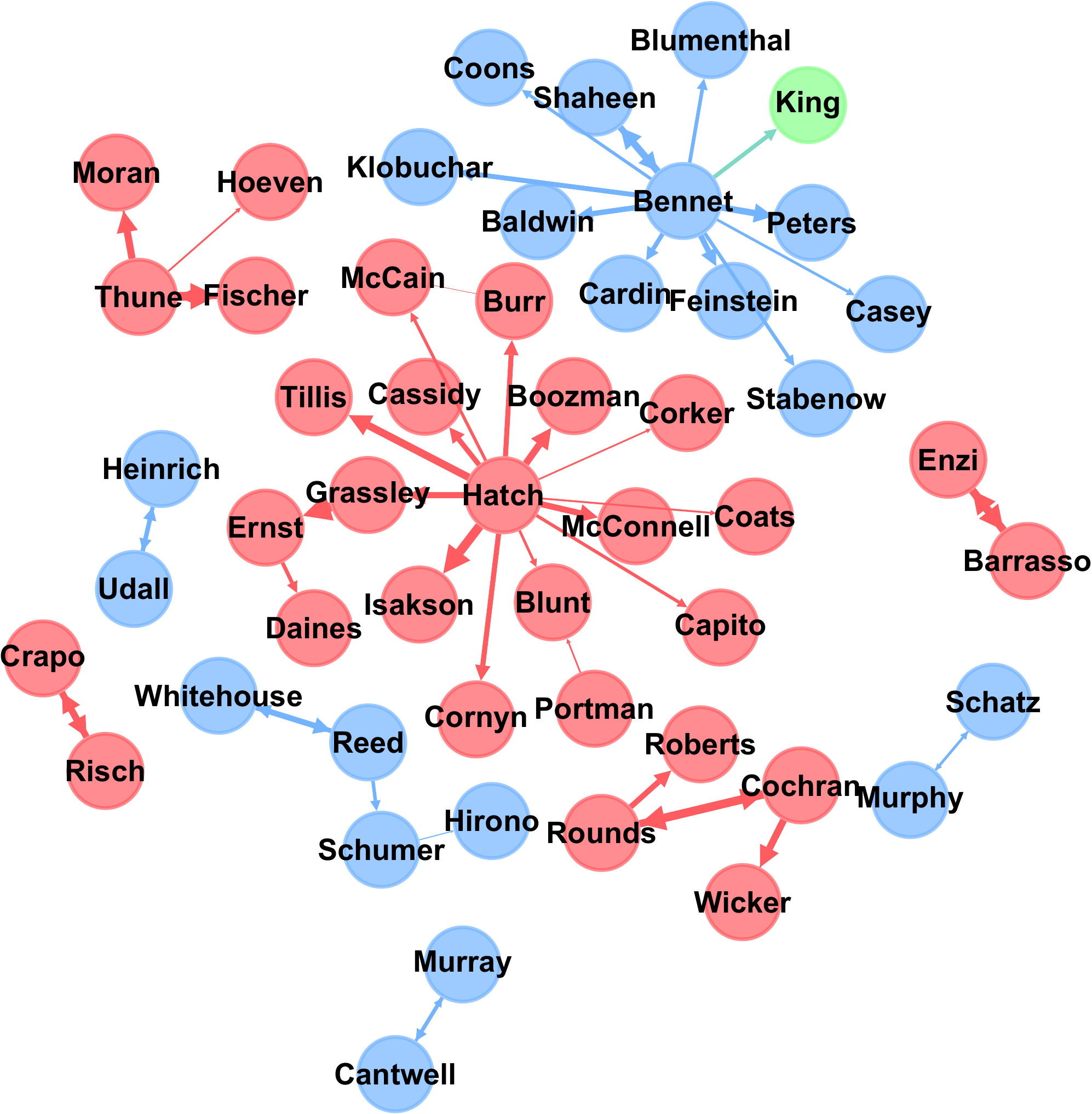} \hspace*{2cm} %
\includegraphics[width=0.4\linewidth]{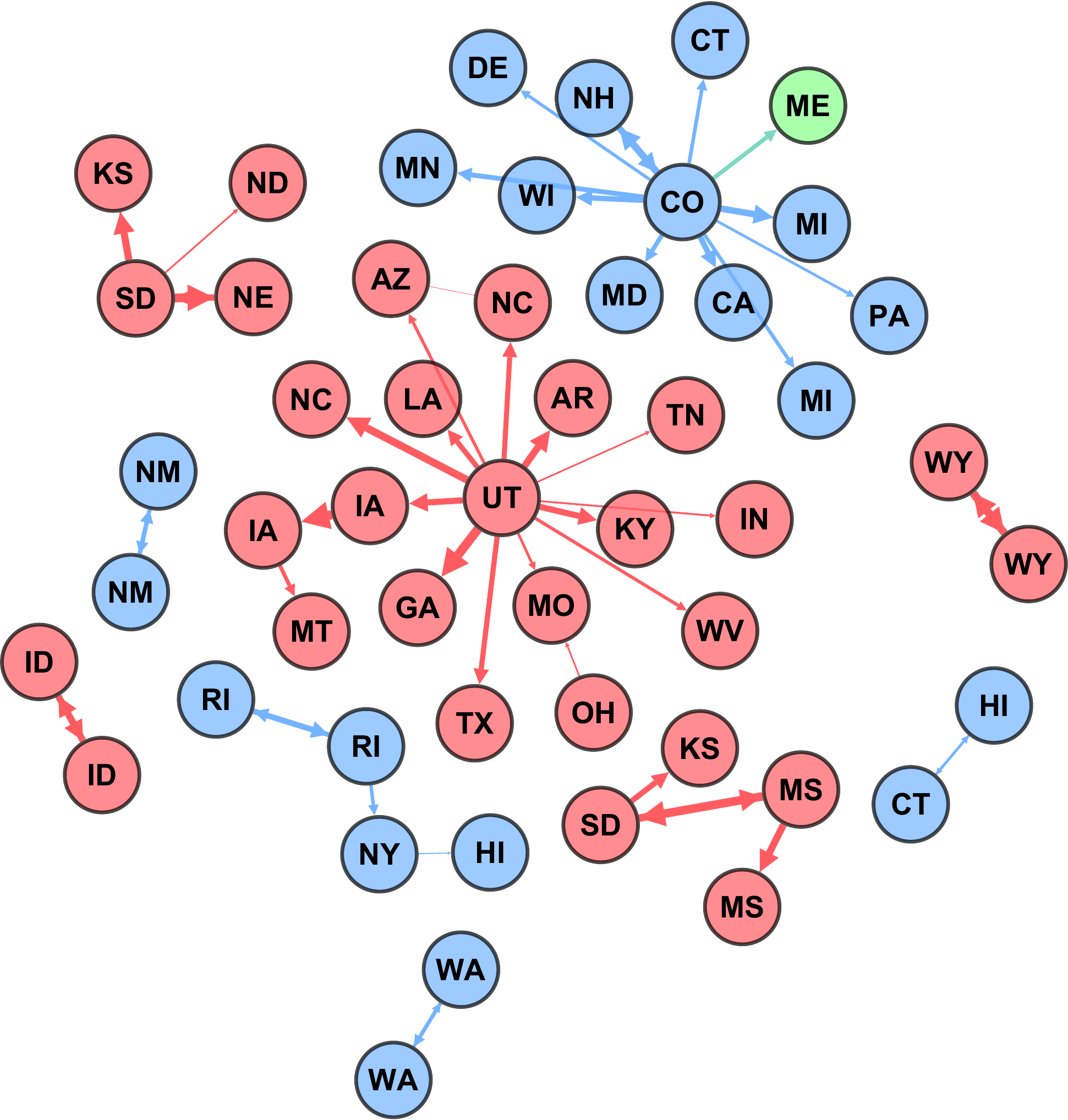}
\caption{The game graph learned from 114th U.S. congressional voting records. Only nodes with degree
greater than one are shown. Colors represent the following: \Lib{Democrat}, \Cons{Republican}, \Neu{Independent}.
The graph on the right shows the states that the senators belong to. The thickness of the edges denote the amount
of influence, computed as the mean absolute payoff, between the senators. Only nodes with degree at least 1 are shown.
\label{fig:sen_graph}}
\end{figure}

\begin{table}[htbp]
\centering
\setlength{\tabcolsep}{0pt}
\resizebox{0.9\textwidth}{!}{%
\begin{tabular}{cccccccccccc}
\hline
\Lib{Baldwin} & \Lib{Bennet} & \Lib{Blumenthal} & \Lib{Cardin} & \Lib{Casey}
 & \Lib{Coons} & \Lib{Feinstein} & \Lib{King} & \Lib{Klobuchar} & \Lib{Peters} & \Lib{Shaheen} & \Lib{Stabenow} \\
\hline
1 & 1 & 1 & 1 & 1 & 1 & 1 & 1 & 1 & 1 & 1 & 1 \\
\hline
\end{tabular}%
}\\
\resizebox{0.9\textwidth}{!}{%
\setlength{\tabcolsep}{0pt}
\begin{tabular}{cccc}
\hline
\Cons{Cochran} & \Cons{Roberts} & \Cons{Rounds} & \Cons{Wicker} \\
\hline
1 & 1 & 1 & 1 \\
\hline
\end{tabular}%
\setlength{\tabcolsep}{0pt}
\begin{tabular}{cccc}
\hline
\Cons{Fischer} & \Cons{Hoeven} & \Cons{Moran} & \Cons{Thune} \\ 
\hline
1 & 1 & 1 & 1 \\
3 & 3 & 3 & 3 \\
\hline
\end{tabular}%
\setlength{\tabcolsep}{0pt}
\begin{tabular}{cccc}
\hline
\Lib{Hirono} & \Lib{Reed} & \Lib{Schumer} & \Lib{Whitehouse} \\
\hline
1 & 1 & 1 & 1 \\
3 & 3 & 3 & 3 \\
\hline
\end{tabular}%
}\\
\resizebox{0.9\textwidth}{!}{%
\setlength{\tabcolsep}{0pt}
\begin{tabular}{ccccccccccccccccc}
\hline
\Cons{Blunt} & \Cons{Boozman} & \Cons{Burr} & \Cons{Capito} & \Cons{Cassidy} & \Cons{Coats} & \Cons{Corker}
& \Cons{Cornyn} & \Cons{Daines} & \Cons{Ernst} & \Cons{Grassley} & 
\Cons{Hatch} & \Cons{Isakson} & \Cons{McCain} & \Cons{McConnell} & \Cons{Portman} & \Cons{Tillis} \\
\hline 
1 & 1 & 1 & 1 & 1 & 1 & 1 & 1 & 1 & 1 & 1 & 1 & 1 & 1 & 1 & 1 & 1 \\
1 & 1 & 1 & 1 & 1 & 1 & 1 & 1 & 1 & 1 & 1 & 1 & 1 & 1 & 1 & 2 & 1 \\
1 & 1 & 1 & 1 & 1 & 1 & 1 & 1 & 1 & 1 & 1 & 1 & 1 & 1 & 1 & 3 & 1 \\
1 & 1 & 1 & 1 & 1 & 1 & 1 & 1 & 1 & 1 & 1 & 2 & 1 & 1 & 1 & 1 & 1 \\
1 & 1 & 1 & 1 & 1 & 1 & 1 & 1 & 1 & 1 & 1 & 2 & 1 & 1 & 1 & 2 & 1 \\
1 & 3 & 3 & 3 & 1 & 3 & 3 & 3 & 1 & 1 & 3 & 3 & 1 & 3 & 3 & 1 & 3 \\
3 & 1 & 1 & 1 & 1 & 1 & 1 & 1 & 1 & 1 & 1 & 2 & 1 & 1 & 1 & 3 & 1 \\
3 & 3 & 3 & 3 & 1 & 3 & 3 & 3 & 1 & 1 & 3 & 3 & 1 & 3 & 3 & 2 & 3 \\
3 & 3 & 3 & 3 & 1 & 3 & 3 & 3 & 1 & 1 & 3 & 3 & 1 & 3 & 3 & 3 & 3 \\
\hline
\end{tabular}%
}
\caption{The PSNE set for the major connected components in the game graph learned from congressional
voting records. The combined number of Nash equilibria computed across senators
with degree at least 1 was 144 and the price of anarchy was computed to be 2.6297.
 \label{tab:sen_psne}}
\end{table}

\subsection{United Nations voting data}
In our final real-world experiment we analyzed roll-call votes in the U.N. General Assembly. The data set
contained votes of 193 countries for 847 U.N. resolutions \footnote{ The data set can be downloaded from
\url{https://dataverse.harvard.edu/dataset.xhtml?persistentId=hdl:1902.1/12379}.}. Each vote could take
one of the three values in $\Set{1,2,3}$, with 1 denoting ``yes'', 2 denoting ``abstention'', and 3 denoting ``no''.
The game graph learned from the data set is shown in Figure \ref{fig:un_graph} while the PSNE set is shown in
Table \ref{tab:un_psne}. As evident from Figure \ref{fig:un_graph} our method recovered two major connected components:
the first consisting of members of the Arab League, and the second consisting of majorly Southeast Asian countries and a few 
other Caribbean islands. The PSNE set once again comprised of strategy profiles where the overwhelming members of a connected
component voted ``yes''. Within the component corresponding to the Arab league, Saudi Arabia, U.A.E., and Bahrain made up a small
coalition of countries that voted identically in the PSNE set.
\begin{figure}[htbp]
\centering
\includegraphics[width=0.6\textwidth]{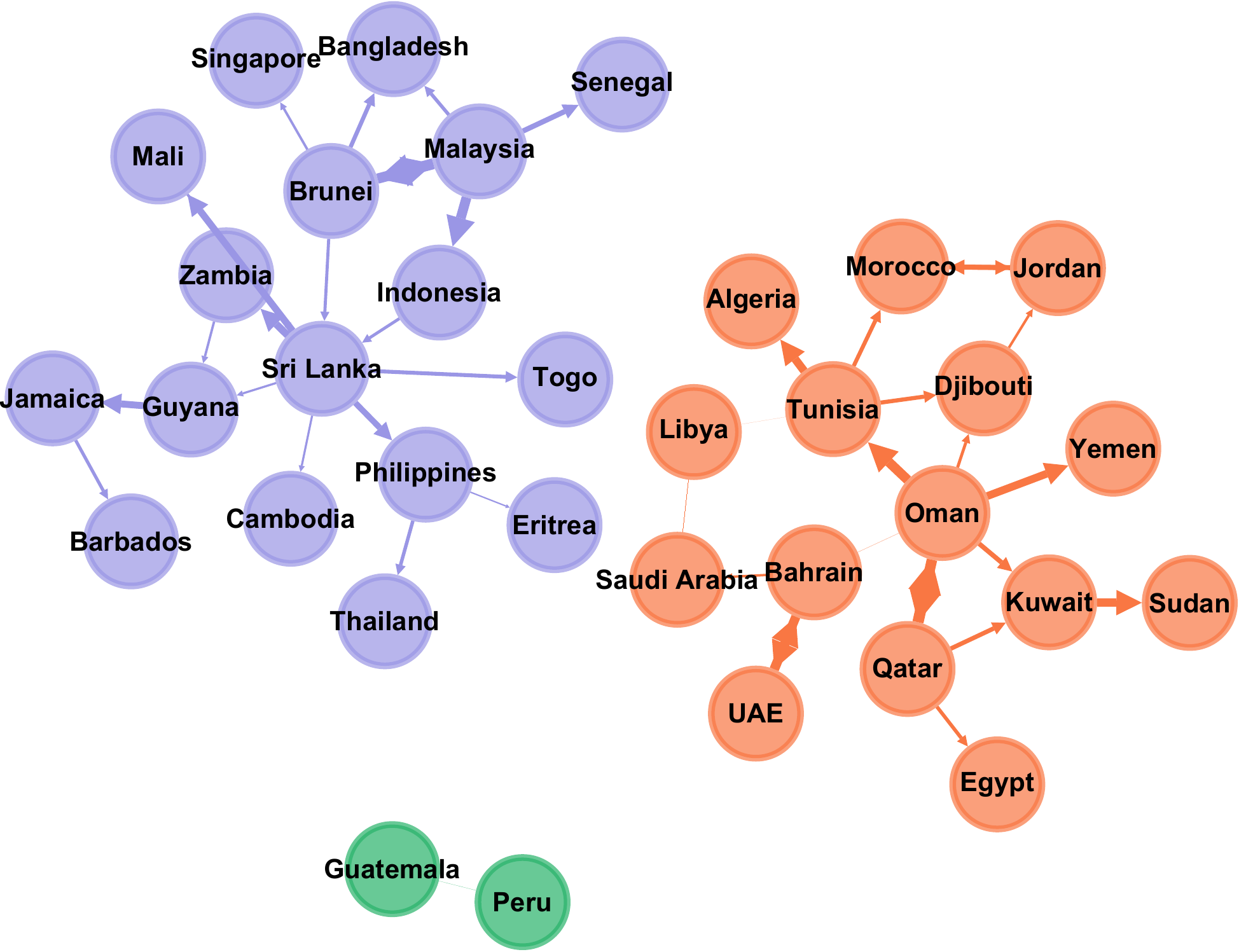}
\caption{The game graph learned from United Nations voting data set. Nodes belonging
to the same connected component have the same color. Only countries with degree at least 1 are shown.
\label{fig:un_graph}}
\end{figure}%
\begin{table}[htbp]
\resizebox{\textwidth}{!}{%
\setlength{\tabcolsep}{2pt}
\begin{tabular}{ccccccccccccccc}
\hline
Algeria & Bahrain & Djibouti & Egypt & Jordan & Kuwait & Libya & Morocco & Oman & Qatar & Saudi Arabia & Sudan & Tunisia & UAE & Yemen \\
\hline
1 & 1 & 1 & 1 & 1 & 1 & 1 & 1 & 1 & 1 & 1 & 1 & 1 & 1 & 1 \\
1 & 2 & 1 & 1 & 1 & 1 & 1 & 1 & 1 & 1 & 2 & 1 & 1 & 2 & 1 \\
\hline
\end{tabular}
}
\resizebox{\textwidth}{!}{%
\setlength{\tabcolsep}{2pt}
\begin{tabular}{ccccccccccccccccc}
\hline
Barbados & Bangladesh & Brunei & Cambodia & Eritrea & Guyana & Indonesia & Jamaica & Malaysia & Mali & Philippines & Senegal & Singapore & Sri Lanka & Thailand & Togo & Zambia \\
\hline
1 & 1 & 1 & 1 & 1 & 1 & 1 & 1 & 1 & 1 & 1 & 1 & 1 & 1 & 1 & 1 & 1 \\ 
1 & 2 & 2 & 1 & 1 & 1 & 2 & 1 & 2 & 1 & 1 & 1 & 2 & 1 & 1 & 1 & 1 \\
1 & 2 & 2 & 2 & 2 & 1 & 2 & 1 & 2 & 2 & 2 & 1 & 2 & 2 & 2 & 2 & 2 \\
2 & 2 & 2 & 2 & 2 & 2 & 2 & 2 & 2 & 2 & 2 & 1 & 2 & 2 & 2 & 2 & 2 \\
\hline
\end{tabular}
}
\caption{The PSNE set for the two major connected components in the game graph learned from
United Nations voting data set. The total number of PSNE was 24 and the price of anarchy was computed to be 3.07.
\label{tab:un_psne}}
\end{table}

\end{appendices}

\end{document}